\documentclass[twoside]{article}

\usepackage[accepted]{aistats2019}

\usepackage{bbm}
\usepackage{fancyhdr}
\usepackage{amsfonts,epsfig,graphicx}
\usepackage{amsmath,amssymb,amsthm} 
\usepackage[T1]{fontenc} 
\usepackage{epsf} 
\usepackage{graphics} 
\usepackage{amsfonts}
\usepackage{natbib} 
\usepackage{color,etoolbox}

\usepackage{url}
\usepackage[inline]{enumitem}
\usepackage{graphicx} 
\usepackage{subfigure} 
\usepackage{verbatim}
\usepackage{booktabs}

\usepackage{algorithm}
\usepackage{algpseudocode}
\usepackage{wrapfig}
\usepackage{lipsum}

\usepackage{amsthm}
\usepackage{mysymbols}

\newcommand{\tparam}{{\mathbf{w}^*}}
\newcommand{\x}{\mathbf{x}}
\newcommand{\z}{\mathbf{z}}
\newcommand{\w}{\mathbf{w}}
\newcommand{\cb}{\mathbf{c}}
\newcommand{\bdelta}{\boldsymbol{\delta}}
\newcommand{\balpha}{\boldsymbol{\alpha}}
\newcommand{\bgamma}{\boldsymbol{\gamma}}

\begin{document}

\twocolumn[

\aistatstitle{Revisiting Adversarial Risk}

\aistatsauthor{ Arun Sai Suggala \And Adarsh Prasad \And Vaishnavh Nagarajan \And  Pradeep Ravikumar }

\aistatsaddress{ Carnegie Mellon University\\ \{asuggala, adarshp, vaishnavh, pradeepr\}@cs.cmu.edu} ]

\begin{abstract}
Recent work on adversarial perturbations shows that there is an inherent trade-off between standard test accuracy and adversarial accuracy. Specifically, it is shown that no classifier can simultaneously be robust to adversarial perturbations and  achieve high standard test accuracy. However, this is contrary to the standard notion that on tasks such as image classification, humans are robust classifiers with low error rate. In this work, we show that the main reason behind this confusion is the inexact definition of adversarial perturbation that is used in the literature.  To fix this issue, we propose a slight, yet important modification to the existing definition of adversarial perturbation.  Based on the modified definition, we show that there is no trade-off between adversarial and standard accuracies; there exist classifiers that are robust and achieve high standard accuracy. We further study several properties of this new definition of adversarial risk and its relation to the existing definition.   
\end{abstract}

\section{Introduction}
\label{sec:intro}
Recent works 
have shown that the output of deep neural networks is vulnerable to even a small amount of perturbation to the input \citep{goodfellow2014explaining, szegedy2013intriguing}. These  perturbations, usually referred to as ``adversarial'' perturbations, are imperceivable by humans and can deceive even state-of-the-art models to make incorrect predictions. Consequently, a line of work in deep learning has focused on defending against such attacks/perturbations~\citep{goodfellow2014explaining,carlini2016towards, ilyas2017robust, madry2017towards}. This has resulted in several techniques for learning models that are robust to adversarial attacks. However, 
many of these techniques were later shown to be ineffective~\citep{athalye2017synthesizing, carlini2017adversarial, 2018arXiv180200420A}. 

We present a brief review of existing literature on adversarial robustness, that is necessarily incomplete. Existing works define an adversarial perturbation at a point $\x$, for a classifier $f$ as any perturbation $\bdelta$ with a small norm, measured w.r.t some distance metric, which changes the output of the classifier; that is $f(\x+\bdelta) \neq f(\x)$. Most of the existing techniques for learning robust models minimize the following worst case loss over all possible perturbations
\begin{equation}
\label{eqn:adv_obj}
\mathbb{E}_{(\x,y) \sim P}\left[\max_{\bdelta: \norm{\bdelta}{} \leq \epsilon} \ell(f(\x+\bdelta), y)\right].
\end{equation} 
\citet{goodfellow2014explaining, carlini2017adversarial, madry2017towards} use heuristics to approximately minimize the above objective. In each iteration of the optimization, these techniques first use heuristics to approximately solve the inner maximization problem and then compute a descent direction using the resulting maximizers. \citet{tsuzuku2018lipschitz} provide a training algorithm which tries to find large margin classifiers with small Lipschitz constants, thus ensuring robustness to adversarial perturbations. A recent line of work has focused on optimizing an upper bound of the above objective. \citet{raghunathan2018certified,zico} provide SDP and LP based upper bound relaxations of the objective, which can be solved efficiently for small networks. These techniques have the added advantage that they can be used to formally verify the robustness of any given model.  
\citet{sinha2017certifiable} propose to optimize the following distributional robustness objective, which is a stronger form of robustness than the one used in  Equation~\eqref{eqn:adv_obj}
\begin{equation}
\label{eqn:dist_robust}
\min_{f} \sup_{Q: W(P, Q) \leq \epsilon} \mathbb{E}_{(\x,y) \sim Q}\left[\ell(f(\x), y)\right],
\end{equation}
where $W(P,Q)$ is the Wasserstein distance between probability distributions $P,Q$.

Another line of work on adversarial robustness has focused on studying adversarial risk from a theoretical perspective. Recently, \citet{schmidt2018adversarially, bubeck2018adversarial} study the generalization properties of adversarial risk and compare it with the generalization properties of standard risk ($\mathbb{P}(y \neq f(\x))$). \citet{fawzi2018analysis,2018arXiv180208686F,franceschi2018robustness} study the properties of adversarial perturbations and adversarial risk. These works characterize the robustness at a point $\x$ in terms of how much perturbation a classifier can tolerate at a point, without changing its prediction
\begingroup\makeatletter\def\f@size{9.5}\check@mathfonts
\begin{align}\label{eqn:fawzidef}
\quad r(\x) = \min_{\bdelta \in \calS } \norm{\bdelta}{} \ \st \ \text{sign}(f(\x)) \neq \text{sign}(f(\x+\bdelta)),
\end{align}
\endgroup
where $\calS$ is some subspace. 
\citet{fawzi2018analysis} theoretically study the expected adversarial radius ($\mathbb{E}[r(\x)]$) of any classifier $f$ and suggest that there is a trade-off between adversarial robustness and the standard accuracy. Specifically, their results suggest that if the prediction accuracy is high then $\mathbb{E}[r(\x)]$ could be small. 

However, these results are contrary to the standard notion that on tasks such as image classification, humans are robust classifiers with low error rate. A careful inspection of the definition of adversarial perturbation and adversarial radius used in Equations~\eqref{eqn:adv_obj},\eqref{eqn:fawzidef} brings into light the inexactness of these definitions.
For example, consider the definition of adversarial risk in Equation~\eqref{eqn:adv_obj}. A major issue with this definition is that it assumes the label $y$ remains the same in a neighborhood of $\x$, and penalizes any classifier which doesn't output $y$ in the entire neighborhood of $\x$. However, the response variable need not remain the same in the neighborhood of $\x$. If a perturbation $\bdelta$ is such that ``true label'' at $\x$ is not the same as the ``true label'' at $\x+\bdelta$ then the classifier shouldn't be penalized for not predicting $y$ at $\x+\bdelta$. Moreover, such a perturbation shouldn't be considered as adversarial, since it changes the true label at $\x+\bdelta$. Figure~\ref{fig:example} illustrates this phenomenon on MNIST and CIFAR-10. As we show later in the paper,  this inexact definition of adversarial perturbation has resulted in recent works claiming that there exists a trade-off between adversarial and standard risks. 

To be more concrete, consider two points $(\x,1)$ and $(\x+\bdelta,-1)$ which are close to each other (\textit{i.e.,} $\|\bdelta\| \leq \epsilon$). Then for any classifier to be correct at the two points, it has to change its prediction for the two points over a small region, which would mean that the adversarial radius, $r(\x)$, is very small. This shows that in order to have high accuracy, a classifier will have to change its score over a small region, leading to a small adversarial radius. This creates the illusion of a trade-off between adversarial robustness and standard risk. This illusion arises because of the above definitions of adversarial perturbation which consider the perturbation $\bdelta$ at $\x$ to be adversarial. On the contrary, $\bdelta$ shouldn't be considered adversarial because the true label at $\x+\bdelta$ is not the same as the label at $\x$. This confusion motivates the need for a clear definition of an adversarial perturbation, the corresponding adversarial risk, and then studying these quantities. 

\begin{figure}[t!]
\subfigure[\small{MNIST ($L_0$)}]{
\centering
\includegraphics[width=0.1\textwidth]{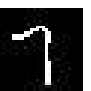}
\includegraphics[width=0.1\textwidth]{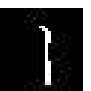}}
\subfigure[\small{CIFAR-10 ($L_2$)}]{
\centering
\includegraphics[width=0.1\textwidth]{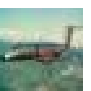}
\includegraphics[width=0.1\textwidth]{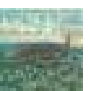}
}
\vspace{-0.1in}
\caption{Images from \citet{sharif2018suitability} showing that small perturbations can change the \emph{true label} of the perturbed image, and which thus should perhaps not be viewed as ``adversarial'' perturbations. Left and right images in each sub-figure correspond to the original and perturbed images respectively. $4.5\%$ of the pixels are corrupted by the $L_0$ adversary and $L_2$ norm bound $\epsilon = 6$ for the $L_2$ adversary.}
\label{fig:example}
\vspace{-0.2in}
\end{figure}

\textbf{Contributions.} In this work, we first formally define the notions of adversarial perturbation,  adversarial risk, which address the above described issue with the existing definition of adversarial risk. Next, we present two key sets of results. 
One set of results pertain to our modified definition of adversarial risk (Sections~\ref{sec:base_bayes},~\ref{sec:smaller_spaces}). In Section~\ref{sec:base_bayes} we show that the minimizers of both adversarial and standard training objectives are Bayes optimal classifiers. This shows that there is no trade-off between adversarial and standard risks and there exist classifiers which have low adversarial and standard risks. Despite this result, in Section~\ref{sec:smaller_spaces}, we show that there is a need for adversarial training.
The second set of results in Section~\ref{sec:base_stoc} analyze the existing definition of adversarial risk to answer some natural questions that come up in light of our results in Section~\ref{sec:base_bayes}. Specifically, we study the conditions under which similar results as in Section~\ref{sec:base_bayes} hold for the existing definition of adversarial risk. 

\section{Preliminaries}
\label{sec:preliminaries}

In this section, we set up the notation and review necessary background on risk minimization. To simplify the presentation in the paper, we only consider the binary classification problem. However, it is straightforward to extend the results and analysis in this paper to multi-class classification.  

Let $(\x, y) \in \mathbb{R}^d \times \{-1, 1\}$ denote the covariate, label pair which follows a probability distribution $P$. Let $S_n = \{(\x_i, y_i)\}_{i = 1}^n$ be $n$ i.i.d samples drawn from $P$. Let $f: \mathbb{R}^d \to \mathbb{R}$ denote a score based classifier, which assigns $\x$ to class $1$, if $f(\x) > 0$. We define the population and empirical risks of classifier $f$ as
\begin{align*}
    R_{0-1}(f) = \mathbb{E}_{(\x,y)\sim P}\left[\ell_{0-1}(f(\x), y)\right],\\
    R_{n,0-1}(f) =  \frac{1}{n}\sum_{i = 1}^n \ell_{0-1}(f(\x_i), y_i),
\end{align*}
where $\ell_{0-1}(\cdot,\cdot)$ is defined as \mbox{$\ell_{0-1}(f(\x),y) = \mathbb{I}(\text{sign}(f(\x)) \neq y)$,}  and $\text{sign}(\alpha) = 1$ if $\alpha > 0$ and $-1$ otherwise. 
Given $S_n$, the objective of empirical risk minimization (ERM) is to estimate a classifier with low population risk $R(f)$. Since optimization of $0/1$ loss is computationally intractable, it is often replaced with a convex surrogate loss function $\ell(f(\x), y) = \phi(yf(\x))$, where $\phi:\mathbb{R} \to [0,\infty)$. Logistic loss is a popularly used surrogate loss and is defined as $\ell(f(\x),y) = \log(1+e^{-yf(\x)})$.
We let $R(f), R_n(f)$ denote the population and empirical risk functions obtained by replacing $\ell_{0-1}$ with $\ell$ in $R_{0-1}(f), R_{n,0-1}(f)$. 

A score based classifier $f^*$ is called Bayes optimal classifier if $\text{sign}(f^*(\x)) = \text{sign}(2P(y = 1|\x) - 1)$ a.e. on the support of distribution $P$. We call $\eta(\x) = \text{sign}(f^*(\x))$ as Bayes decision rule. Note that the Bayes decision rule need not be unique.
We assume that the set of points where $P(y = 1|\x) = \frac{1}{2}$ has measure $0$.


\section{Adversarial Risk}
\label{sec:adv_perturbation}

In this paper, we focus on the following robustness setting, which is also the focus of most of the past works on adversarial robustness: given a pre-trained model, there is an adversary which corrupts the inputs to the model such that the corrupted inputs lead to certain ``unwanted'' behavior in the model. Our goal is to design models that are robust to such adversaries. In what follows, we make the notions of an adversary, unwanted behavior more concrete and formally define adversarial perturbation and adversarial risk.

Let $\mathcal{A}:\mathbb{R}^d \to \mathbb{R}^d$ be an adversary which modifies any given data point $\x$ to $\mathcal{A}(\x)$. Let $\bdelta_{\x} = \mathcal{A}(\x) - \x$ be the perturbation chosen by the adversary at $\x$. We assume that the perturbations are norm bounded, which is a standard restriction imposed on the capability of the adversary.

Our definition of adversarial perturbation is based on a reference or a base classifier. For example, in vision tasks, this base classifier is the human vision system.
A perturbation is adversarial to a classifier if it changes the prediction of the classifier, whereas the base/reference classifier assigns it to the same class as the unperturbed point.
\begin{definition}[Adversarial Perturbation]
Let $f:\mathbb{R}^d \to \mathbb{R}$ be a score based classifier and \mbox{$g:\mathbb{R}^d \rightarrow \{-1, 1\}$} be a base classifier.  Then the perturbation $\bdelta_{\x}$ chosen by an adversary $\mathcal{A}$ at $\x$ is said to be adversarial for $f$, w.r.t base classifier $g$, if $\|\bdelta_{\x}\| \leq \epsilon$ and 
\[
\text{sign}(f(\x)) = g(\x), \quad g(\x) = g(\x+\bdelta_{\x}),
\]
and 
\[
\text{sign}(f(\x+\bdelta_{\x})) \neq g(\x).
\]
Equivalently, a perturbation $\bdelta_{\x}$ is said to be adversarial for  $f$, w.r.t base classifier $g$, if
$\|\bdelta_{\x}\| \leq \epsilon$, \mbox{$g(\x) = g(\x+\bdelta_{\x})$} and 
\[
\ell_{0-1}\left(f(\x+\bdelta_{\x}), g(\x) \right) - \ell_{0-1}\left(f(\x), g(\x) \right) = 1.
\]
\end{definition}
Note that, unlike the existing notion of adversarial risk, the above definition doesn't consider a perturbation as adversarial if it changes the label of the base classifier. Moreover, if $f$ disagrees with $g$ at $\x$, then the perturbation $\bdelta_{\x}$ is not considered adversarial. This is reasonable because if $f(\x)$ disagrees with $g(\x)$, it should be treated as a standard classification error rather than adversarial error. Using the above definition of adversarial perturbation, we next define adversarial risk.
\begin{definition}[Adversarial Risk]
Let $f:\mathbb{R}^d \to \mathbb{R}$ be a score based classifier and $g:\mathbb{R}^d \rightarrow \{-1, 1\}$ be a base classifier. The adversarial risk of $f$ w.r.t  base classifier $g$ and adversary $\mathcal{A}$ is defined as the fraction of points which can be adversarially perturbed by $\mathcal{A}$
\begingroup\makeatletter\def\f@size{9}\check@mathfonts
\[
R_{\text{adv}, 0-1}(f) = \mathbb{E}\left[ \ell_{0-1}\left(f(\x+\bdelta_{\x}), g(\x) \right) - \ell_{0-1}\left(f(\x), g(\x) \right) \right].
\]
\endgroup
\end{definition}
It is typically assumed that the adversary $\mathcal{A}$ is an ``optimal'' adversary; that is, at any give point $\x$, $\mathcal{A}$ tries to find a perturbation that is adversarial for $f$
\begingroup\makeatletter\def\f@size{9}\check@mathfonts
\[
\bdelta_{\x} \in \argmax_{\begin{subarray}{c} \|\bdelta\| \leq \epsilon \\  g(\x) = g(\x+\bdelta) \end{subarray}} \ell_{0-1}\left(f(\x +\bdelta), g(\x) \right) - \ell_{0-1}\left(f(\x), g(\x) \right).
\]
\endgroup
The adversarial risk of a classifier $f$ w.r.t an optimal adversary can then be written as
\begingroup\makeatletter\def\f@size{8.8}\check@mathfonts
\[
\mathbb{E}\left[\max_{\begin{subarray}{c} \|\bdelta\| \leq \epsilon \\  g(\x) = g(\x+\bdelta) \end{subarray}} \ell_{0-1}\left(f(\x + \bdelta), g(\x) \right) - \ell_{0-1}\left(f(\x), g(\x) \right) \right].
\]
\endgroup
In the sequel, we assume that the adversary is optimal and work with the above definition of adversarial risk. 

Let $R_{\text{adv}}(f)$ denote the adversarial risk obtained by replacing $\ell_{0-1}$ with a convex surrogate loss $\ell$ and let $R_{n,\text{adv}}(f)$ denote its empirical version.
In the sequel we refer to $R(f), R_{\text{adv}}(f)$ as standard and adversarial risks and $R_n(f)$, $R_{n,\text{adv}}(f)$ as the corresponding empirical risks. 
The goal of adversarial training is to learn a classifier that has low adversarial and standard risks. One natural technique to estimate such a robust classifier is to minimize a linear combination of both the risks
\begin{equation}
\label{eqn:objective_mod}
\argmin_{f \in \mathcal{F}} R(f)+\lambda R_{\text{adv}}(f),
\end{equation}
where $\mathcal{F}$ is an appropriately chosen function class and $\lambda \geq 0$ is a hyper-parameter. 
The tuning parameter $\lambda$ trades off standard risk with the \emph{excess risk} incurred from adversarial perturbations, and allows us to tune the conservativeness of our classifier.  


\section{Bayes Optimal Classifier as Base Classifier}
\label{sec:base_bayes}
In this section we study the properties of minimizers of objective~\eqref{eqn:objective_mod}, under the assumption that the base classifier $g(\x)$ is a Bayes optimal classifier. 
This is a reasonable assumption because if we are interested in robustness with respect to a base classifier, it is likely we are getting labels from the base classifier itself. For instance, in many classification tasks the labels are generated by humans (\textit{i.e.,} human is a Bayes optimal classifier for the classification task) and robustness is also measured w.r.t a human. 
The following Theorem shows that under this condition, the minimizers of \eqref{eqn:objective_mod} are Bayes optimal. 
\begin{theorem}
\label{lem:minimizers}
Suppose the hypothesis class $\mathcal{F}$ is the set of all measurable functions. Let the base classifier $g(x)$ be a Bayes optimal classifier.
\begin{enumerate}
\itemsep0em 
    \item (\textbf{$0/1$ loss}). If $\ell$ is the $0/1$ loss, then any minimizer $\hat{f}$ of $$\min_{f \in \mathcal{F}} R_{0-1}(f)+\lambda R_{\text{adv}, 0-1}(f),$$ is a Bayes optimal classifier.
    \item (\textbf{Logistic loss}). Suppose $\ell$ is the logistic loss and suppose the probability distribution $P$ is such that $\left|P(y = 1 | \x) - \frac{1}{2}\right| > \gamma$ a.e., for some positive constant $\gamma$. Then any minimizer of Equation~\eqref{eqn:objective_mod} is a Bayes optimal classifier. 
\end{enumerate}
\end{theorem}\vspace{-0.05in}
The first part of the above Theorem shows that minimizing the joint objective with $0/1$ loss, for any choice of $\lambda \geq 0$, results in a Bayes optimal classifier. This shows that there exist classifiers that are both robust and achieve high standard accuracy and there is no trade-off between adversarial and standard risks. More importantly, the Theorem shows that if there exists a unique Bayes decision rule (\textit{i.e.,} $\text{sign}(f_1^*) = \text{sign}(f_2^*)$ a.e. for any two Bayes optimal classifiers $f_1, f_2$), then standard training suffices to learn robust classifiers and there is no need for adversarial training. 

The second part of the Theorem, which is perhaps the more interesting result, shows that using a convex surrogate for the $0/1$ loss to minimize the joint objective also results in Bayes optimal classifiers. This result assures us that optimizing a convex surrogate does not hinder our search for a robust classifier that has low adversarial and standard risks.  Finally, we note that the requirement on conditional class probability $P(y|\x)$ is a mild condition as $\gamma$ can be any small positive constant close to 0.

\subsection{Approximate Bayes Optimal Classifier as Base Classifier}
We now briefly discuss the scenario where the base classifier $g(\x)$ is not Bayes optimal. In this setting, the minimizers of the objective~\eqref{eqn:objective_mod} need not be Bayes optimal. The first term in the objective will bias the optimization towards a Bayes optimal classifier.  Whereas, the second term in the joint objective will bias the optimization towards the base classifier. Since the base classifier is not a Bayes optimal classifier, this results in a trade-off between the two terms, which is controlled by the tuning parameter $\lambda$. If $\lambda$ is small, then the minimizers of the joint objective will be close to a Bayes optimal classifier. If $\lambda$ is large, the minimizers will be close to the base classifier. 

\section{Old definition of Adversarial Risk}
\label{sec:base_stoc}
 One natural question that Section~\ref{sec:base_bayes} gives rise to is whether the results in Theorem~\ref{lem:minimizers} also hold for the definition of adversarial risk used by the existing works.
To answer this question, we now study the properties of minimizers of the adversarial training objective in Equation~\eqref{eqn:adv_obj}.  We start by making a slight modification to the definition of adversarial risk $R_{\text{adv},0-1}(f)$ and analyzing the minimizers of the resulting adversarial training objective.  Let $H_{\text{adv},0-1}(f)$ be the adversarial risk obtained by removing the constraint $g(\x+\bdelta) = g(\x)$ in $R_{\text{adv},0-1}(f)$
\begingroup\makeatletter\def\f@size{8.5}\check@mathfonts
\[
H_{\text{adv},0-1}(f) = \mathbb{E}\left[ \sup_{\begin{subarray}{c} \|\bdelta\|\leq \epsilon \end{subarray}} \ell_{0-1}\left(f(\x+\bdelta), g(\x) \right) - \ell_{0-1}\left(f(\x), g(\x) \right) \right].
\]
\endgroup
We call this the adversarial ``smooth'' risk, because by removing the constraint, we are implicitly assuming that the base classifier is smooth in the neighborhood of each point. Let $H_{\text{adv}}(f)$ denote the adversarial risk obtained by replacing $\ell_{0-1}$ in $H_{\text{adv},0-1}(f)$ with a convex surrogate loss $\ell$.

The following Theorem studies the minimizers of the adversarial training objective obtained using the adversarial smooth risk.
Specifically, it shows that if there exists a Bayes decision rule which satisfies a ``margin condition'', then minimizing the adversarial training objective using $H_{\text{adv},0-1}(f)$ results in Bayes optimal classifiers.
\begin{theorem}
\label{lem:constraint_removal}
Suppose the hypothesis class $\mathcal{F}$ is the set of all measurable functions. Moreover, suppose the base classifier is a Bayes decision rule $\eta(\x)$ which satisfies the following margin condition: 
\begin{equation}
\label{eqn:margin}
\text{Pr}\left(\left\lbrace \x: \exists \tilde{\x},  \|\tilde{\x} - \x\| \leq \epsilon \text{ and } \eta(\tilde{\x}) \neq \eta(\x)\right\rbrace \right) = 0.
\end{equation}

\begin{enumerate}
\itemsep0em 
    \item (\textbf{$0/1$ loss}). If $\ell$ is the $0/1$ loss, then any minimizer of $R_{0-1}(f)+\lambda H_{\text{adv}, 0-1}(f)$ is a Bayes optimal classifier.
    \item (\textbf{Logistic loss}). Suppose $\ell$ is the logistic loss. Moreover, suppose the probability distribution $P$ is such that $\left|P(y = 1 | \x) - \frac{1}{2}\right| > \gamma$ a.e., for some positive constant $\gamma$. Then any minimizer of 
\begin{equation}
\label{eqn:objective_no_constraint}
\min_{f \in \mathcal{F}} R(f) + \lambda H_{\text{adv}}(f),
\end{equation}
is a Bayes optimal classifier
\end{enumerate}
\end{theorem}\vspace{-0.07in}
The margin condition in Equation~\eqref{eqn:margin} requires the Bayes decision rule to \emph{not} change its prediction in the neighborhood of any given point.
We note that this condition is necessary for the results of the above Theorem to hold.  In Section~\ref{sec:margin_imp} we show that without the margin condition, the minimizers of \eqref{eqn:objective_no_constraint} need not be Bayes optimal. Theorem~\ref{lem:constraint_removal} also highlights the importance of the constraint ``$g(\x+\bdelta) = g(\x)$'' in the definition of adversarial risk, for Bayes optimality of the minimizers. 

\subsection{Replacing Base Classifier with Stochastic Label $y$}
A natural step is to replace $g(\x)$ in the definition of adversarial smooth risk $H_{\text{adv},0-1}(\cdot)$ with stochastic label $y$ and study the properties of minimizers of the resulting objective. Our results show that the resulting adversarial training objective behaves similarly as Equation~\eqref{eqn:objective_no_constraint}.
\begin{theorem}
\label{lem:equiv}
Consider the setting of Theorem~\ref{lem:constraint_removal}. Let $G_{\text{adv},0-1}(f)$ be the adversarial risk obtained by replacing $g(\x)$ with $y$ in $R_{\text{adv},0-1}(f)$
\begingroup\makeatletter\def\f@size{9}\check@mathfonts
\[
G_{\text{adv},0-1}(f) = \mathbb{E}\left[ \sup_{\begin{subarray}{c} \|\bdelta\|\leq \epsilon \end{subarray}} \ell_{0-1}\left(f(\x+\bdelta), y \right) - \ell_{0-1}\left(f(\x), y \right) \right].
\]
\endgroup
\begin{enumerate}
\itemsep0em 
    \item (\textbf{$0/1$ loss}). If $\ell$ is the $0/1$ loss, then any minimizer of $R_{0-1}(f)+\lambda G_{\text{adv}, 0-1}(f)$ is a Bayes optimal classifier.
    \item (\textbf{Logistic loss}). Suppose $\ell$ is the logistic loss. Moreover, suppose the probability distribution $P$ is such that $\left|P(y = 1 | \x) - \frac{1}{2}\right| > \gamma$ a.e., for some positive constant $\gamma$. Then any minimizer of 
    
\begin{equation}
\label{eqn:objective_standard}
\min_{f \in \mathcal{F}} R(f) + \lambda G_{\text{adv}}(f),
\end{equation}
is a Bayes optimal classifier
\end{enumerate}


 
\end{theorem}
Note that, objective~\eqref{eqn:adv_obj} is equivalent to objective~\eqref{eqn:objective_standard} for $\lambda = 1$. The Theorem thus shows that  under the margin condition there is no trade-off between the popularly used definition of adversarial risk and standard risk.

\subsection{Importance of Margin}
\label{sec:margin_imp}
If no Bayes decision rule satisfies the margin condition, then the results of Theorems~\ref{lem:constraint_removal},\ref{lem:equiv} do not hold and  minimizers of the corresponding joint objectives need not be Bayes optimal. 
\begin{theorem}[Necessity of margin]
\label{thm:nomargin}
Consider the setting of Theorem~\ref{lem:constraint_removal}. Suppose no Bayes decision rule satisfies the margin condition in Equation~\eqref{eqn:margin}. Then $\exists \lambda_0$ such that \mbox{$\forall \lambda > \lambda_0$} the minimizers of the joint objectives \mbox{$R_{0-1}(f)+\lambda H_{\text{adv}, 0-1}(f)$} and \mbox{$R_{0-1}(f)+\lambda G_{\text{adv}, 0-1}(f)$} are not Bayes optimal. 
\end{theorem}
The above Theorem shows that without the margin condition, performing adversarial training using existing definition of adversarial risk can possibly result in a loss of standard accuracy.
Next, we consider a concrete example and empirically validate our findings from Theorems~\ref{lem:equiv},~\ref{thm:nomargin}.
\paragraph{Synthetic Dataset.} Consider the following data generation process in a 2D space. Let $S(\cb,r)$ denote the axis aligned square of side length $r$, centered at $\cb$. The marginal distribution of $\x$  follows a uniform distribution on $S([-2,0]^T, 2)\cup S([2,0]^T, 2)$. The conditional distribution of $y$ given $\x$ is given by
$$y|\x \in S([2,0]^T, 2)=\begin{cases}1, \quad &\text{w.p. } 0.7\\-1, \quad &\text{w.p. } 0.3\end{cases},$$
$$y|\x \in S([-2,0]^T, 2)=\begin{cases}1, \  &\text{w.p. } 0.3\\-1, \  &\text{w.p. } 0.7\end{cases}.$$
Note that the data satisfies the margin condition in Equation~\eqref{eqn:margin} w.r.t $L_{\infty}$ norm, for $\epsilon=1$ and the following Bayes decision rule 
\[
\eta(\x) = \begin{cases}1 & \text{if } \x(1) \geq 0\\ -1 & \text{if } \x(1) < 0\end{cases}.
\]
From Theorem~\ref{lem:equiv} we know that for $L_{\infty}$ norm perturbations with $\epsilon \leq 1$, minimizing Equation~\ref{eqn:objective_standard} results in Bayes optimal classifiers. To verify this, we generated $10^5$ training samples from this distribution and minimized objective~\eqref{eqn:objective_standard} over the set of linear classifiers. Since the model is linear, we have a closed form expression for the adversarial risk. Moreover, objective~\eqref{eqn:objective_standard} can be efficiently solved using gradient descent. Figure~\ref{fig:toy} shows the behavior of standard risk of the resulting models as we vary $\epsilon$. We can seen that for $\epsilon \leq 1$, the standard risk is equal to $0.3$, which is the Bayes optimal risk. Whereas, for $\epsilon > 1$, the standard risk can be larger than $0.3$. 
\begin{figure}[t!]
\centering
\includegraphics[width=0.25\textwidth]{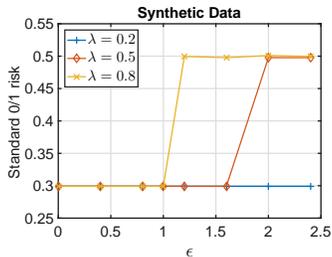}
\vspace{-0.1in}
\caption{Standard $0/1$ risk vs. perturbation radius $\epsilon$ on the synthetic dataset, for estimators minimizing the old adversarial risk with perturbations measured w.r.t $L_{\infty}$ norm. When perturbation radius is more than Bayes rule margin for this problem viz. $1$, and when $\lambda$ is large, the standard risk is larger than the Bayes optimal risk viz. $0.3$. Note that for small $\lambda$, even if the perturbation is larger than margin, the minimizer is Bayes optimal.}
\label{fig:toy}
\end{figure}

\paragraph{Benchmark Datasets.} A number of recent works try to explain the drop in standard accuracy in adversarially trained models~\citep{fawzi2018analysis, tsipras2018there}. These works suggest that there could be an inherent trade-off between standard and adversarial risks. In contrast, our results show that as long as there exists a Bayes optimal classifier with sufficient margin, minimizers of objectives~\eqref{eqn:adv_obj},~\eqref{eqn:objective_standard} have low standard and adversarial risks and there is no trade-off between the two risks. The important question then is, ``Do the benchmark datasets such as MNIST~\citep{lecun1998mnist}, CIFAR10~\citep{krizhevsky2009learning} satisfy the margin condition?'' \citet{sharif2018suitability} try to estimate the margin in MNIST, CIFAR10 datasets via user studies. Their results suggest that for $L_{\infty}$ perturbations larger than what is typically used in practice ($\epsilon = 0.1$), CIFAR10 doesn't not satisfy the margin condition. Together with our results, this shows that for such large perturbations, adversarial training will result in models with low standard accuracy. However, it is still unclear if the benchmark datasets satisfy the margin condition for $\epsilon$ typically used in practice. We believe answering this question can help us understand if it possible to obtain a truly robust model, without compromising on standard accuracy.

\subsection{Standard training with increasing model complexity}
\label{sec_increase_model_complexity}
Before we conclude the section, we show how our results from Theorem~\ref{lem:equiv} can be used to explain an interesting phenomenon observed by \cite{madry2017towards}: even with standard risk minimization, complex networks result in more robust classifiers than simple networks. 
Define the standard and adversarial training objectives as
\[
\textbf{(standard) }\min_{f \in \mathcal{F}} R(f),
\]
\[
\textbf{(adversarial) } \min_{f \in \mathcal{F}} R(f) + \lambda G_{\text{adv}}(f).
\]
%
 
Let $\mathcal{F}$ be a small function class, such as the set of functions which can be represented using a particular neural network architecture. As we increase the complexity of $\mathcal{F}$, we expect the minimizer of the standard risk $R(f)$ to move closer to a Bayes optimal classifier. Assuming the margin condition is satisfied, from Theorem~\ref{lem:equiv} we know that the minimizer of the adversarial training objective is also a Bayes optimal classifier. So as we increase the complexity of $\mathcal{F}$, we expect the minimizer of the standard risk to also have low joint adversarial + standard risk, and the minimizer of the joint adversarial + standard risk to have low standard risk. The latter could thus serve as an explanation for the other empirically observed phenomenon that performing adversarial training on increasingly complex networks results in classifiers with better standard risk. 
Figures~\ref{fig:depth_ERM},~\ref{fig:depth_pgd_gen} illustrate these two phenomena on MNIST and CIFAR10 datasets. To ensure the margin condition is at least approximately satisfied, we use small perturbations in these experiments. More details about the hyper-parameteres used in the experiments can be found in the Appendix.
\begin{figure}[H]
\centering     
\begin{minipage}[b]{0.48\textwidth}
\begin{subfigure}
{\includegraphics[width=0.48\textwidth]{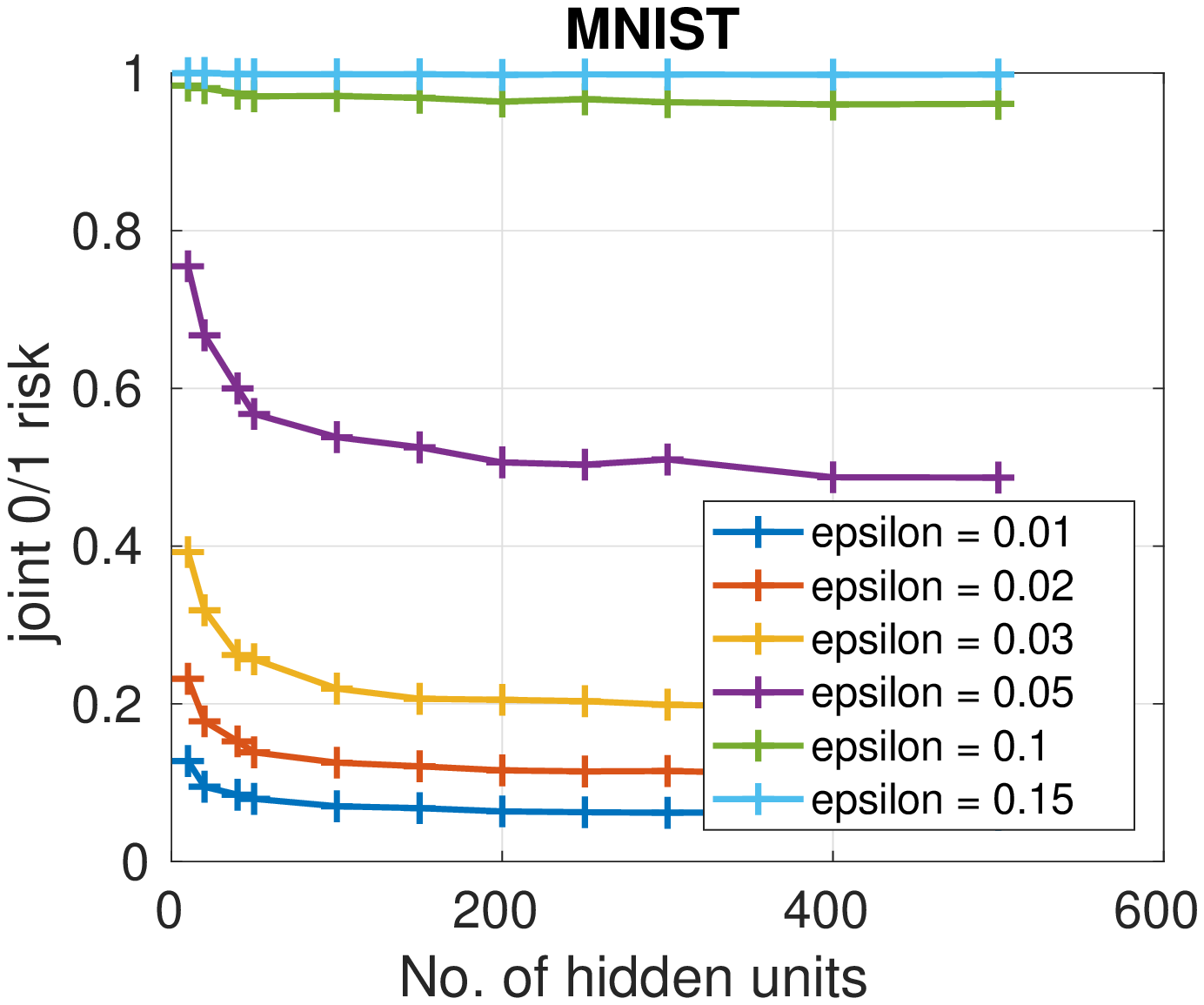}}
{\includegraphics[width=0.48\textwidth]{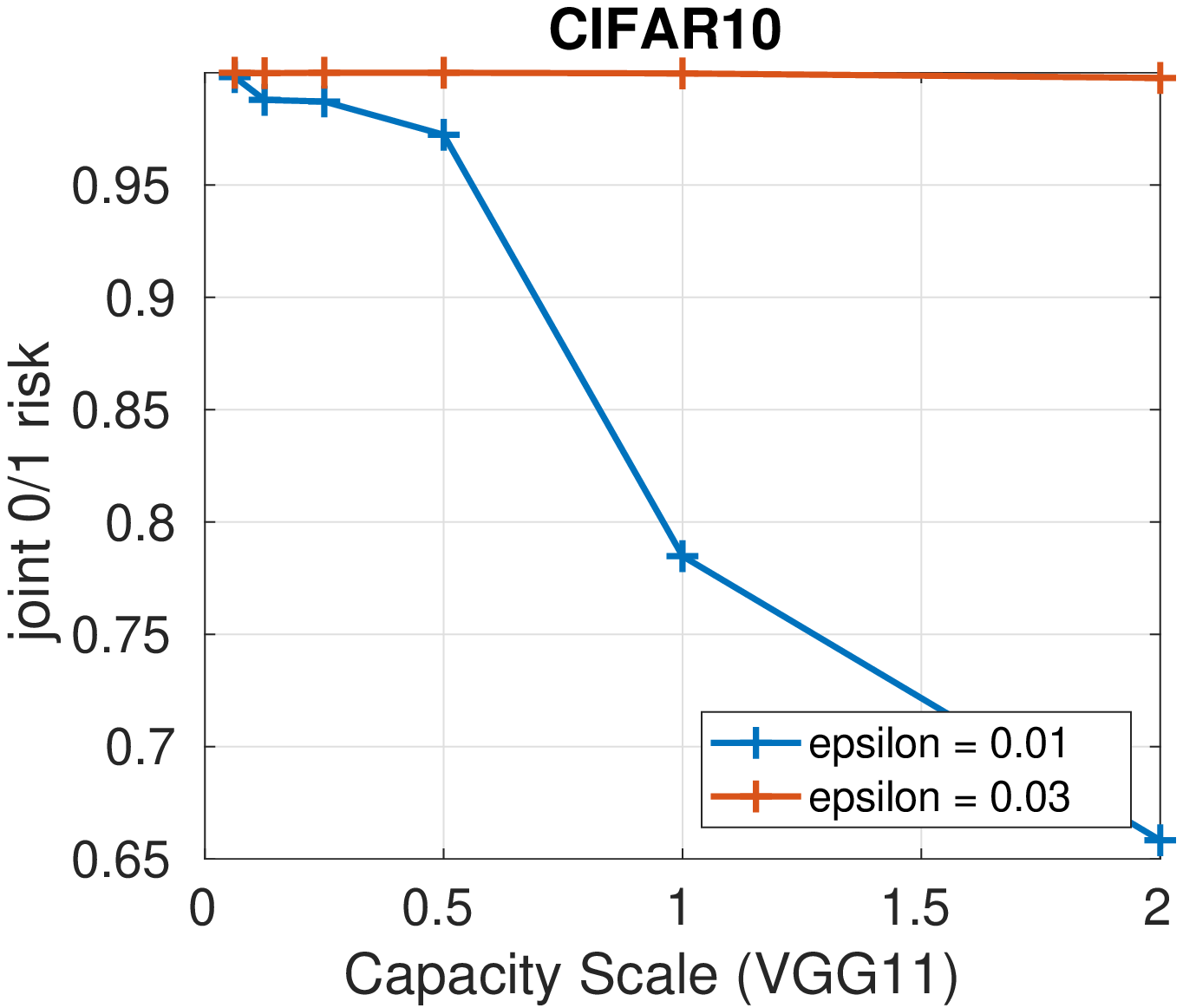}}
\caption{\small{Lowering of adversarial + standard $0/1$ risk (i.e. joint $0/1$ risk) of models obtained through standard training, as we increase the model capacity. }}
\label{fig:depth_ERM}
\end{subfigure}
\end{minipage}
\hspace{0.15in}
\begin{minipage}[b]{0.48\textwidth}
\begin{subfigure}
{\includegraphics[width=0.475\textwidth]{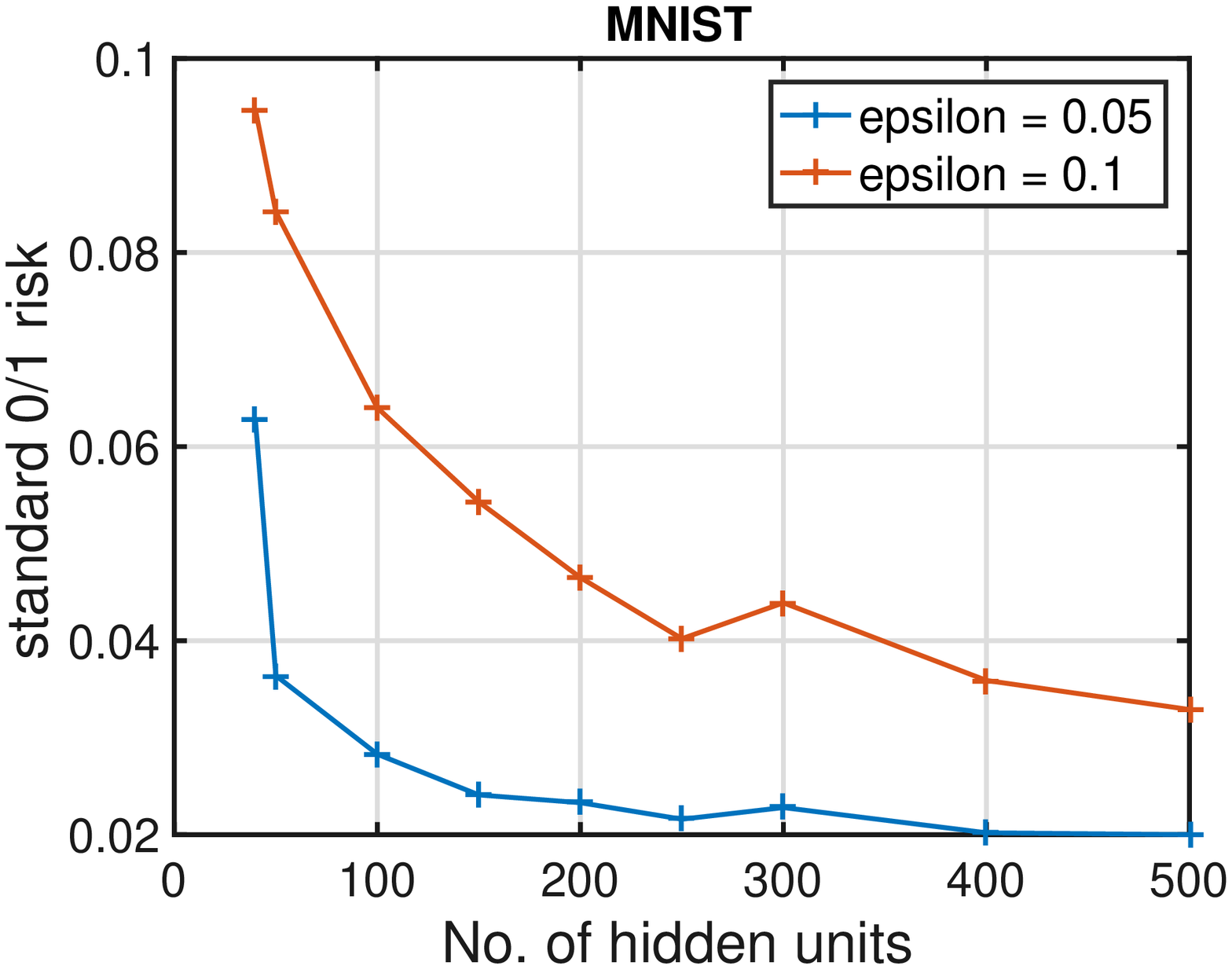}}
{\includegraphics[width=0.485\textwidth]{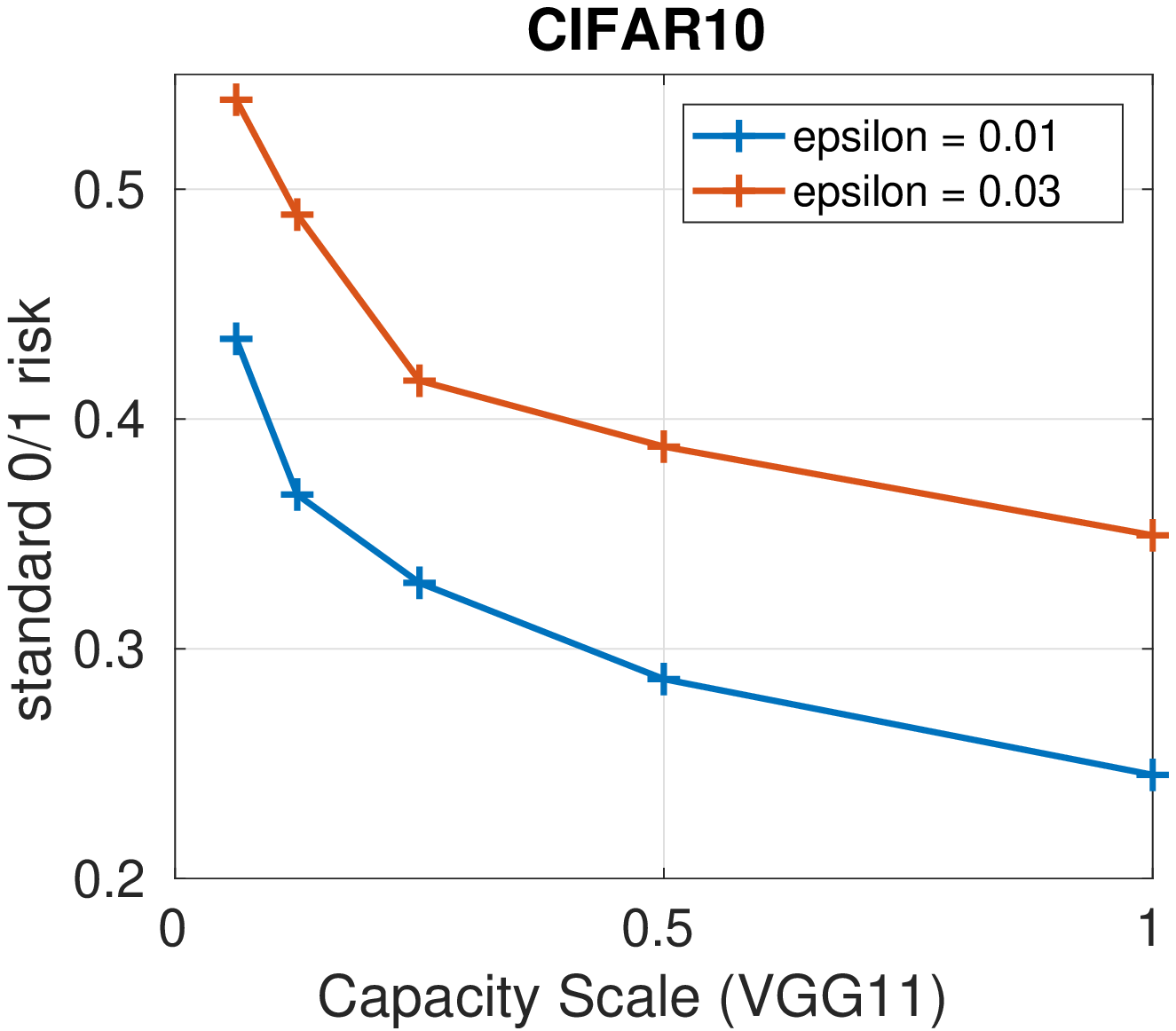}}
\caption{\small{Lowering of standard $0/1$ risk of models obtained through adversarial training (with $\lambda = 1$), as we increase the model capacity.}}
\label{fig:depth_pgd_gen}
\end{subfigure}
\end{minipage}
\end{figure}

We conclude the discussion by pointing out that in practice we optimize empirical risks instead of population risks, so that our explanations above are accurate only when empirical risks and the corresponding population risks have similar landscapes. 

\vspace{-0.1in}

\section{Importance of Adversarial Training}
\label{sec:smaller_spaces}
Recall, in Section~\ref{sec:base_bayes} we studied the properties of adversarial training when the base classifier is Bayes Optimal. In particular, in Theorem~\ref{lem:minimizers} we showed that the minimizers of adversarial training objective~\eqref{eqn:objective_mod} are Bayes optimal classifiers, which are also the minimizers of standard risk. 
This naturally leads us to the following question: Do we really need adversarial training?  Will standard training suffice to learn robust classifiers? In this section, we investigate conditions under which standard risk minimization alone does not guarantee robust classifiers.


 We first consider the setting where there is a single Bayes decision rule. Then, our theoretical results in Theorem~\ref{lem:minimizers} indeed show that there is no need for adversarial training, but provided we are able to find the optimum hypothesis over the set of all measurable functions. However, in practice, we are never able to do so due to the finite amount of data available to us. What if we choose a smaller hypothesis class (such as the set of linear separators)? In Section~\ref{sec:single_bayes} we show that standard risk minimization over restricted hypothesis classes can result in classifiers with low standard risk but high adversarial risk. More generally, in Section~\ref{sec:calibration}, we show that adversarial and standard risks are not calibrated, and that small standard risks need not entail a small adversarial risk.

In Section~\ref{sec:multiple_bayes}, we consider the setting where there are multiple Bayes decision rules. For instance, when the data is separable or lies in a low-dimensional manifold, Bayes decision rule is not unique.
In this setting, even if one has access to unlimited data (which allows us to optimize over the space of all measurable functions), we show that there is a need for adversarial training. Although all the Bayes decision rules have the same standard risk, they can differ on adversarial risk. In such cases, it is impossible to distinguish these classifiers using standard risk. As a result, one needs to perform adversarial training to learn a robust Bayes decision rule.

We study these questions theoretically using a mixture model where the data for each class is generated from a different mixture component. The distribution of $\x$ conditioned on $y$ follows a normal distribution: $\x | y \sim \calN(y \tparam, \sigma^2 \calI_d)$, where $\calI_d \in \mathbb{R}^{d \times d}$ is the identity matrix and $P(y=1) = P(y = -1) = \half$. Note that in this setting $\x  \mapsto  \x^T\tparam$ is a Bayes optimal classifier. Moreover there is a unique Bayes decision rule.

\vspace{-0.05in}
\subsection{Optimizing Standard Risk over Restricted Function Class}\label{sec:single_bayes}
We study the effect of minimizing the standard risk over restricted function classes. Consider the restricted hypothesis class of all vectors which are non-zero in the top-$k$ co-ordinates: $W_k = \{ \w \in \real^d | ~ \w(i) = 0 \ \forall i > k \}$. The following result shows that the \emph{exact} minimizer of standard risk over this restricted hypothesis class need \emph{not} be the minimizer of the adversarial risk over this class, even for perturbations as small as $\frac{1}{\sqrt{d}}$.

\begin{theorem}\label{lem:functionClass}
Consider the Gaussian mixture model with $\tparam = [\frac{1}{\sqrt{d/2}},\frac{1}{\sqrt{d/2}},\ldots,\frac{1}{\sqrt{d/2}}]^T$, $\sigma = 1$ and let $\tilde{\w} = \argmin_{\w \in W_{d/2}} R_{0-1}(f_{\w})$ be the minimizer of the standard risk when restricted to $W_{d/2}$. Then, even for a small enough perturbation of $\epsilon \geq \frac{C}{\sqrt{d}}$ \wrt $L_{\infty}$ norm, we have that
\begingroup\makeatletter\def\f@size{9}\check@mathfonts
\[ R_{0-1}(f_{\tilde{\w}}) - R_{0-1}(f_\tparam) < 0.1 ~~ \text{but} ~~ R_{adv,0-1}(f_{\tilde{\w}}) > 0.95,\]
where $R_{adv,0-1}(\cdot)$ is measured \wrt $g(\x)=\text{sign}(\x^T\tparam)$.
\end{theorem}
\endgroup

\vspace{-0.05in}
\subsection{Standard and Adversarial Risk are not Calibrated}\label{sec:calibration}
We next explore if the two risks are \emph{calibrated}, \ie \emph{does {approximately} minimizing the standard risk always lead to small adversarial risk?} Suppose the mean of the Gaussian components is $k$-sparse; that is, $\tparam$ has $k$ non-zero entries. Then the Bayes optimal classifier $\x \mapsto \x^T\tparam$ depends only on a few features and there are a lot of irrelevant features. The following result shows that there exist linear separators which achieve near-optimal classification accuracy, but have a high adversarial risk, even for a $L_{\infty}$ adversarial perturbation of size $\frac{1}{\sqrt{d-k}}$. 

\begin{theorem}\label{cor:sparse}
Let $\tparam$ be $k$-sparse with non-zeros in the first $k$ coordinates. Let $\w \in \mathbb{R}^d$ be a linear separator such that $\w_{1:k} = \tparam_{1:k}$, $\w_{k+1:d} = [ \frac{\pm1}{\sqrt{d-k}},\ldots,\frac{\pm1}{\sqrt{d-k}}]$. Then, there exists a constant $C$ such that if $\norm{\tparam}{2} \geq C$ and $\sigma =1$, the excess risk of $f_{\w}(\x) = \x^T\w$ is small; that is, $R_{0-1}(f_{\w}) - R^*_{0-1} \leq  0.02$, where $R^*_{0-1}$ is the risk of the Bayes optimal classifier. However, even for a small enough perturbation $\epsilon \geq \frac{2\|\tparam\|^2_2}{\sqrt{d-k}}$ w.r.t $L_{\infty}$ norm, the adversarial risk satisfies \[ R_{adv,0-1}(f_{\w}) \geq 0.95,\]
where the base classifier $g(\x)$ is equal to $\text{sign}(\x^T\tparam)$. 
\end{theorem}
Note that the constructed classifier $\w$ has very small weights on irrelevant features. Hence the classification error is low \emph{but not minimal}. But since there are a lot of such irrelevant features, there exist adversarial perturbations which don't change the prediction of Bayes classifier, but change the prediction of $\w$.

\subsection{Multiple Bayes Decision Rules}\label{sec:multiple_bayes}
In this section, we consider the setting where there could be multiple Bayes optimal decision rules. We consider the question whether different Bayes optimal solutions have different adversarial risks, and whether standard risk minimization gives us robust Bayes optimal solutions. 

Suppose our data comes from low dimensional Gaussians embedded in a high-dimensional space, \ie suppose $\norm{\tparam}{0} = k \ll d$ and the covariance matrix $D$ of the conditional distributions $\x|y$ is diagonal with $i^{th}$ diagonal entry $D_{ii} = \sigma^2~\text{if}~\w_i^* \neq 0, 0~\text{otherwise}$. Notice that in this model any classifier $\tilde\w$ such that $\tilde{\w}_{1:k} = \tparam_{1:k}$ is a Bayes optimal classifier.
Observe the subtle difference between this setting and sparse linear model. In particular, in the previous example, the data is inherently high-dimensional, but with only a few relevant discriminatory features; on the contrary, here the data lies on a low dimensional manifold of a high dimensional subspace.

In this setting, we study the adversarial risk of classifiers obtained through minimization of $R(f_{\w})$ using iterative methods such as gradient descent.

\begin{theorem}\label{cor:surr}
Let $\tparam$ be such that $\norm{\tparam}{2} \geq C$, for some constant $C$. Let $\epsilon \geq \frac{2\|\tparam\|^2_2}{\sqrt{d-k}}$ and $\ell$ be any convex calibrated surrogate loss $\ell(f_\w(\x), y) = \phi(y\w^T\x)$. Then gradient descent on $R(f_\w)$ with random initialization using a Gaussian distribution with covariance $\frac{1}{\sqrt{d-k}}\calI_d$ converges to a point $\hat{\w}_{GD}$ such that with high probability,
\[ 
 R_{0,1}(f_{\hat{\w}_{GD}}) = 0~~ \text{but} ~~ R_{adv,0-1}(f_{\hat{\w}_{GD}}) \geq 0.95, \]
 where $R_{adv,0-1}(\cdot)$ is the adversarial risk measured \wrt $\tparam$.
\end{theorem}
Note that Theorem~\ref{cor:surr} raises the vulnerability of standard risk minimization using gradient descent by showing that it can lead to Bayes optimal solutions which have high adversarial risk. Moreover, observe that increasing $d$ results in classifiers that are less robust; even a $O(1/\sqrt{d-k})$ perturbation can create adversarial examples with respect to $\tparam$. 
All our results in this section show that standard risk minimization is inherently insufficient in providing robustness. This suggests the need for adversarial training.
\section{Regularization properties of Adversarial Training}
In this section, we study the regularization properties of the adversarial training objective in Equation~\eqref{eqn:objective_mod}. Specifically, we show that the adversarial risk $R_{\text{adv}}(f)$, effectively acts as a regularizer which biases the solution towards certain classifiers.  
The following Theorem explicitly shows this regularization effect of adversarial risk. 
\begin{theorem}
\label{lem:reg_new}
Let $\|.\|_*$ be the dual norm of $\|.\|$, which is defined as:  $\|\z\|_* = \sup_{\|\x\| = 1} \z^T\x.$  Suppose $\ell$ is the logistic loss and suppose the classifier $f$ is differentiable a.e. 
Then for any $\epsilon \geq 0$ the adversarial training objective~\eqref{eqn:objective_mod}  can be upper bounded as
\begingroup\makeatletter\def\f@size{9}\check@mathfonts
\begin{equation*}
\begin{array}{l}
\displaystyle R(f) + \lambda R_{\text{adv}}(f) \leq \vspace{0.1in}\\
\quad R(f)  + \lambda \min\left\lbrace \epsilon\mathbb{E}\left[ \sup_{\|\bdelta\|\leq \epsilon}\|\nabla f(\x+\bdelta)\|_*  \right], 2\|f - g\|_{\infty}\right\rbrace,
\end{array}
\end{equation*}
\endgroup
where $\|f - g\|_{\infty} = \sup_{\x}|f(\x) - g(\x)|$.
\end{theorem}
Although the above Theorem only provides an upper bound, it still provides insights into the regularization effects of adversarial risk.
It shows that adversarial risk effectively acts as a regularization term biasing the optimization towards two kinds of classifiers: 1) classifiers that are smooth with small gradients and 2) classifiers that are pointwise close to the base classifier $g(\x)$. 
We now compare the regularization effect of adversarial risk in objective~\eqref{eqn:objective_mod} with the regularization effect of existing notion of adversarial risk. 
\begin{theorem}
\label{lem:reg_old}
Suppose $\ell$ is the logistic loss and suppose the classifier $f$ is differentiable a.e. 
Then for any $\epsilon \geq 0$ the adversarial training objective~\eqref{eqn:objective_standard} can be upper bounded as
\begingroup\makeatletter\def\f@size{9}\check@mathfonts
\begin{equation*}
\begin{array}{l}
\displaystyle R(f) + \lambda G_{\text{adv}}(f) \leq  R(f)  + \lambda \epsilon\mathbb{E}\left[ \sup_{\|\bdelta\|\leq \epsilon}\|\nabla f(\x+\bdelta)\|_*  \right].
\end{array}
\end{equation*}
\endgroup
Moreover, for linear classifiers $f_{\w}(\x) = \w^T\x$, the adversarial training objective~\eqref{eqn:objective_standard}  can be upper and lower bounded as
\begingroup\makeatletter\def\f@size{9}\check@mathfonts
\begin{equation}
\begin{array}{l}
\displaystyle R(f_{\w}) + \lambda G_{\text{adv}}(f_{\w}) \leq  R(f_{\w})  + \lambda \epsilon\|\w\|_*,
\end{array}
\end{equation}
\[
\displaystyle R(f_{\w}) + \lambda G_{\text{adv}}(f_{\w}) \geq  R(f_{\w})  + \left(\frac{\lambda \epsilon}{2} R_{0-1}(f_{\w})\right)\|\w\|_*.
\]
\endgroup
\end{theorem}
Comparing Theorems~\ref{lem:reg_new},~\ref{lem:reg_old}, we can see that the major difference between the two adversarial risks is that the existing definition doesn't necessarily bias the optimization towards the base classifier $g(\x)$, whereas the new definition certainly biases the optimization towards $g(\x)$. 

For linear classifiers, the above Theorem provides a tight upper bound and shows that adversarial training using objective~\eqref{eqn:objective_standard} essentially acts as a regularizer which penalizes the dual norm of $\w$. In a related work, \citet{xu2009robustness} focus on linear classifiers with hinge loss, and show that under separability conditions on the data and certain additional constraint on perturbations, the robust objective is equivalent to the regularized objective.

\vspace{-0.1in}

\section{Summary and Future Work}
\label{sec:conclusion}
In this work, we identified the inaccuracies with the existing definition of adversarial risk and proposed a new definition of adversarial risk which fixes these inaccuracies. We analyzed the properties of minimizers of the resulting adversarial training objective and showed that Bayes optimal classifiers are its minimizers and that there is no trade-off between adversarial and standard risks. 
We also studied the existing definition of adversarial risk, its relation to the new definition, and identify conditions under which its minimizers are Bayes optimal. Our analysis highlights the importance of margin for Bayes optimality of its minimizers. 

An important direction for future work would be to design algorithms for minimization of the new adversarial training objective.  One can consider two different approaches in this direction: 1) assuming we have black box access to the base classifier, one could  design efficient optimization techniques which make use of the black box. 2) assuming we have access to an approximate base classifier (\textit{e.g.,} some complex model which is pre-trained on a lot of labeled data from a related domain, or a ``teacher'' network), one could use this classifier as a surrogate for the base classifier, to optimize the adversarial training objective. 

\section{Acknowledgements}
\label{sec:ack}
We acknowledge the support of NSF via IIS-1149803, IIS-1664720, DARPA via FA87501720152, and PNC via the PNC Center for Financial Services Innovation.

\bibliography{local}

\begin{thebibliography}{22}
\providecommand{\natexlab}[1]{#1}
\providecommand{\url}[1]{\texttt{#1}}
\expandafter\ifx\csname urlstyle\endcsname\relax
  \providecommand{\doi}[1]{doi: #1}\else
  \providecommand{\doi}{doi: \begingroup \urlstyle{rm}\Url}\fi

\bibitem[Goodfellow et~al.(2014)Goodfellow, Shlens, and
  Szegedy]{goodfellow2014explaining}
Ian~J Goodfellow, Jonathon Shlens, and Christian Szegedy.
\newblock Explaining and harnessing adversarial examples.
\newblock \emph{arXiv preprint arXiv:1412.6572}, 2014.

\bibitem[Szegedy et~al.(2013)Szegedy, Zaremba, Sutskever, Bruna, Erhan,
  Goodfellow, and Fergus]{szegedy2013intriguing}
Christian Szegedy, Wojciech Zaremba, Ilya Sutskever, Joan Bruna, Dumitru Erhan,
  Ian Goodfellow, and Rob Fergus.
\newblock Intriguing properties of neural networks.
\newblock \emph{arXiv preprint arXiv:1312.6199}, 2013.

\bibitem[Carlini and Wagner(2016)]{carlini2016towards}
Nicholas Carlini and David Wagner.
\newblock Towards evaluating the robustness of neural networks.
\newblock \emph{arXiv preprint arXiv:1608.04644}, 2016.

\bibitem[Ilyas et~al.(2017)Ilyas, Jalal, Asteri, Daskalakis, and
  Dimakis]{ilyas2017robust}
Andrew Ilyas, Ajil Jalal, Eirini Asteri, Constantinos Daskalakis, and
  Alexandros~G Dimakis.
\newblock The robust manifold defense: Adversarial training using generative
  models.
\newblock \emph{arXiv preprint arXiv:1712.09196}, 2017.

\bibitem[Madry et~al.(2017)Madry, Makelov, Schmidt, Tsipras, and
  Vladu]{madry2017towards}
Aleksander Madry, Aleksandar Makelov, Ludwig Schmidt, Dimitris Tsipras, and
  Adrian Vladu.
\newblock Towards deep learning models resistant to adversarial attacks.
\newblock \emph{arXiv preprint arXiv:1706.06083}, 2017.

\bibitem[Athalye and Sutskever(2017)]{athalye2017synthesizing}
Anish Athalye and Ilya Sutskever.
\newblock Synthesizing robust adversarial examples.
\newblock \emph{arXiv preprint arXiv:1707.07397}, 2017.

\bibitem[Carlini and Wagner(2017)]{carlini2017adversarial}
Nicholas Carlini and David Wagner.
\newblock Adversarial examples are not easily detected: Bypassing ten detection
  methods.
\newblock In \emph{Proceedings of the 10th ACM Workshop on Artificial
  Intelligence and Security}, pages 3--14. ACM, 2017.

\bibitem[{Athalye} et~al.(2018){Athalye}, {Carlini}, and
  {Wagner}]{2018arXiv180200420A}
A.~{Athalye}, N.~{Carlini}, and D.~{Wagner}.
\newblock {Obfuscated Gradients Give a False Sense of Security: Circumventing
  Defenses to Adversarial Examples}.
\newblock \emph{ArXiv e-prints}, February 2018.

\bibitem[Tsuzuku et~al.(2018)Tsuzuku, Sato, and Sugiyama]{tsuzuku2018lipschitz}
Yusuke Tsuzuku, Issei Sato, and Masashi Sugiyama.
\newblock Lipschitz-margin training: Scalable certification of perturbation
  invariance for deep neural networks.
\newblock \emph{arXiv preprint arXiv:1802.04034}, 2018.

\bibitem[Raghunathan et~al.(2018)Raghunathan, Steinhardt, and
  Liang]{raghunathan2018certified}
Aditi Raghunathan, Jacob Steinhardt, and Percy Liang.
\newblock Certified defenses against adversarial examples.
\newblock \emph{arXiv preprint arXiv:1801.09344}, 2018.

\bibitem[Kolter and Wong(2017)]{zico}
J.~Zico Kolter and Eric Wong.
\newblock Provable defenses against adversarial examples via the convex outer
  adversarial polytope.
\newblock \emph{CoRR}, abs/1711.00851, 2017.
\newblock URL \url{http://arxiv.org/abs/1711.00851}.

\bibitem[Sinha et~al.(2017)Sinha, Namkoong, and Duchi]{sinha2017certifiable}
Aman Sinha, Hongseok Namkoong, and John Duchi.
\newblock Certifiable distributional robustness with principled adversarial
  training.
\newblock \emph{arXiv preprint arXiv:1710.10571}, 2017.

\bibitem[Schmidt et~al.(2018)Schmidt, Santurkar, Tsipras, Talwar, and
  M{\k{a}}dry]{schmidt2018adversarially}
Ludwig Schmidt, Shibani Santurkar, Dimitris Tsipras, Kunal Talwar, and
  Aleksander M{\k{a}}dry.
\newblock Adversarially robust generalization requires more data.
\newblock \emph{arXiv preprint arXiv:1804.11285}, 2018.

\bibitem[Bubeck et~al.(2018)Bubeck, Price, and
  Razenshteyn]{bubeck2018adversarial}
S{\'e}bastien Bubeck, Eric Price, and Ilya Razenshteyn.
\newblock Adversarial examples from computational constraints.
\newblock \emph{arXiv preprint arXiv:1805.10204}, 2018.

\bibitem[Fawzi et~al.(2018)Fawzi, Fawzi, and Frossard]{fawzi2018analysis}
Alhussein Fawzi, Omar Fawzi, and Pascal Frossard.
\newblock Analysis of classifiers' robustness to adversarial perturbations.
\newblock \emph{Machine Learning}, 107\penalty0 (3):\penalty0 481--508, 2018.

\bibitem[{Fawzi} et~al.(2018){Fawzi}, {Fawzi}, and
  {Fawzi}]{2018arXiv180208686F}
A.~{Fawzi}, H.~{Fawzi}, and O.~{Fawzi}.
\newblock {Adversarial vulnerability for any classifier}.
\newblock \emph{ArXiv e-prints}, February 2018.

\bibitem[Franceschi et~al.(2018)Franceschi, Fawzi, and
  Fawzi]{franceschi2018robustness}
Jean-Yves Franceschi, Alhussein Fawzi, and Omar Fawzi.
\newblock Robustness of classifiers to uniform l and gaussian noise.
\newblock \emph{arXiv preprint arXiv:1802.07971}, 2018.

\bibitem[Sharif et~al.(2018)Sharif, Bauer, and Reiter]{sharif2018suitability}
Mahmood Sharif, Lujo Bauer, and Michael~K Reiter.
\newblock On the suitability of lp-norms for creating and preventing
  adversarial examples.
\newblock \emph{arXiv preprint arXiv:1802.09653}, 2018.

\bibitem[Tsipras et~al.(2018)Tsipras, Santurkar, Engstrom, Turner, and
  Madry]{tsipras2018there}
Dimitris Tsipras, Shibani Santurkar, Logan Engstrom, Alexander Turner, and
  Aleksander Madry.
\newblock There is no free lunch in adversarial robustness (but there are
  unexpected benefits).
\newblock \emph{arXiv preprint arXiv:1805.12152}, 2018.

\bibitem[LeCun(1998)]{lecun1998mnist}
Yann LeCun.
\newblock The mnist database of handwritten digits.
\newblock \emph{http://yann. lecun. com/exdb/mnist/}, 1998.

\bibitem[Krizhevsky and Hinton(2009)]{krizhevsky2009learning}
Alex Krizhevsky and Geoffrey Hinton.
\newblock Learning multiple layers of features from tiny images.
\newblock Technical report, Citeseer, 2009.

\bibitem[Xu et~al.(2009)Xu, Caramanis, and Mannor]{xu2009robustness}
Huan Xu, Constantine Caramanis, and Shie Mannor.
\newblock Robustness and regularization of support vector machines.
\newblock \emph{Journal of Machine Learning Research}, 10\penalty0
  (Jul):\penalty0 1485--1510, 2009.

\end{thebibliography}
\bibliographystyle{unsrtnat}
\clearpage

\onecolumn
\appendix

\section{Proof of Theorem~\ref{lem:minimizers}}
\subsection{Intermediate Results}
Before we present the proof of the Theorem, we present useful intermediate results which we require in our proof.  The following Lemmas present some monotonicity properties of the logistic loss.
\begin{lemma}
\label{lem:logistic_monotone}
Let $y$ be a discrete random variable such that $$y = \begin{cases}1, \quad & \text{with probability} \geq \frac{1}{2} + \gamma\\ -1, \quad & \text{with probability} \leq \frac{1}{2} - \gamma\end{cases},$$ for some $\gamma > 0$. Let $\xi = \log\frac{1+2\gamma}{1-2\gamma}$ and let $z< \xi$ be a constant. Define $h(u)$ as follows
\[
h(u) = \mathbb{E}_y[\log(1+e^{-y((1-u)z+u\xi)})].
\]
Then $h(u)$ is a strictly decreasing function over the domain $[0,1)$.
\end{lemma}
\begin{proof}
Let $p=P(y=1)$. The derviative of $h(u)$, w.r.t $u$ is given by
\[
h'(u) = p\times \frac{(z-\xi)}{1+e^{(1-u)z+u\xi}}+(1-p)\times \frac{(\xi-z) e^{(1-u)z+u\xi} }{1+e^{(1-u)z+u\xi}}.
\]
We will now show that $h'(u) < 0$. Suppose $p < 1$ (otherwise it is easy to see that $h'(u) < 0$). Then
\begin{equation*}
    \begin{array}{lll}
         \left(\frac{1+e^{(1-u)z+u\xi}}{\xi-z}\right) \times h'(u)  &=&  - p + (1-p)e^{(1-u)z+u\xi}\\
         & =& (1-p)\left(e^{(1-u)z+u\xi}-\frac{p}{1-p}\right)\\
         &\leq& (1-p)\left(e^{(1-u)z+u\xi}-e^{\xi}\right)\\
         &=& (1-p)e^{\xi}\left(e^{(1-u)(z-\xi)}-1\right)\\
         &<& 0.
    \end{array}
\end{equation*}
\end{proof}
\begin{lemma}
\label{lem:logistic_monotone2}
Let $u, \xi$ be such that $\xi > 0, u\in [0,1)$. Define functions $h_1(z), h_2(z)$ as follows
\[
h_1(z) = \log(1+e^{-(1-u)z-u\xi})-\log(1+e^{-z}).
\]
\[
h_2(z) = \log(1+e^{+(1-u)z+u\xi})-\log(1+e^{z}).
\]
Then $ h_1(z)$ is an increasing function over the domain $(-\infty, \xi)$ and $h_2(z)$ is a decreasing function over $(-\infty, \xi)$.
\end{lemma}
\begin{proof}
The derivative of $h_1(z)$ w.r.t $z$ is given by
\[
h'_1(z) = -\frac{1-u}{1+e^{(1-u)z+u\xi}} + \frac{1}{1+e^{z}}.
\]
We will now show that $h'_1(z) \geq 0$.
\begin{equation*}
    \begin{array}{lll}
         h'_1(z) &=& -\frac{1-u}{1+e^{(1-u)z+u\xi}} + \frac{1}{1+e^{z}}  \\
         & \geq & -\frac{1}{1+e^{(1-u)z+u\xi}} + \frac{1}{1+e^{z}} \\
         &\geq & -\frac{1}{1+e^{z}} + \frac{1}{1+e^{z}}\\
         &=& 0
    \end{array},
\end{equation*}
where the first inequality follows from the fact that $u \in [0,1)$ and the second inequality follows from the fact that $z < \xi$. This shows that $h_1$ is increasing over $(-\infty, \xi)$.

We use a similar argument to show that $h_2'(z)$ is a decreasing function. Consider the derivative of $h_2(z)$ w.r.t $z$
\[
h'_2(z) = \frac{1-u}{1+e^{-(1-u)z-u\xi}} - \frac{1}{1+e^{-z}}.
\]
We will now show that $h'_2(z) \leq 0$.
\begin{equation*}
    \begin{array}{lll}
         h'_2(z) &=& \frac{1-u}{1+e^{-(1-u)z-u\xi}} - \frac{1}{1+e^{-z}}  \\
         & \leq & \frac{1}{1+e^{-(1-u)z-u\xi}} - \frac{1}{1+e^{-z}} \\
         &\leq & \frac{1}{1+e^{-z}} - \frac{1}{1+e^{-z}}\\
         &=& 0
    \end{array},
\end{equation*}
This shows that $h_2$ is decreasing over $(-\infty, \xi)$.
\end{proof}
\subsection{Main Argument}
\paragraph{$\mathbf{0/1}$ loss.} We first prove the Theorem for $0/1$ loss; that is, we show that any minimizer of $R_{0-1}(f) + \lambda R_{\text{adv}, 0-1}(f)$ is a Bayes optimal classifier. 
 We prove the result by contradiction. Let $f^*$ be a Bayes optimal classifier such that $\text{sign}(f^*(\x)) = g(\x)$ a.e. Suppose $\hat{f}$ is a minimizer of the joint objective. Let $\text{sign}(\hat{f}(\x))$ disagree with  $\text{sign}(f^*(\x))$ over a set $X$ of non-zero measure. We show that the joint risk of $\hat{f}$ is strictly larger than $f^*$.
 
 First, we show that the standard risk of $\hat{f}$ is strictly larger than $f^*$:
\begin{equation*}
    \begin{array}{lll}
         R_{0-1}(\hat{f}) - R_{0-1}(f^*) &=& \mathbb{E}_{(\x, y)}\left[\ell_{0-1}(\hat{f}(\x), y) - \ell_{0-1}(f^*(\x), y)\right]     \\
         & =& \text{Pr}(\x \in X)\times \mathbb{E}_{(\x,y)}\left[\ell_{0-1}(\hat{f}(\x), y) - \ell_{0-1}(f^*(\x), y) \Big| \x \in  X\right]\\
         &=& \text{Pr}(\x \in X)\times \mathbb{E}_{\x}\left[\mathbb{E}_y\left[\ell_{0-1}(\hat{f}(\x), y) - \ell_{0-1}(f^*(\x), y) \Big|\x\right]\Big| \x \in  X\right]\\
         &=& \text{Pr}(\x \in X)\times \mathbb{E}_{\x}\left[P(y \neq \text{sign}(\hat{f}(\x))|\x) - P(y \neq \text{sign}(f^*(\x))|\x)\Big| \x \in  X\right]\\
         &>& 0,
    \end{array}
\end{equation*}
where the last inequality follows from the definition of Bayes optimal decision rule.

We now show that the adversarial risk of $\hat{f}$ is larger than $f^*$. Since the base classifier $g$ agrees with $f^*$ a.e. we have
\[
R_{\text{adv}, 0-1}(f^*) =  \mathbb{E}\left[ \max_{\begin{subarray}{c} \|\bdelta\|\leq \epsilon \\  g(\x) = g(\x+\bdelta) \end{subarray}} \ell_{0-1}\left(f^*(\x+\bdelta), g(\x) \right) - \ell_{0-1}\left(f^*(\x), g(\x) \right) \right] = 0.
\]
Since $R_{\text{adv}, 0-1}$ of any classifier is always non-negative, this shows that $R_{\text{adv}, 0-1}(\hat{f}) \geq R_{\text{adv}, 0-1}(f^*)$.  Combining this with the above result on classification risk we get
\[
R_{0-1}(\hat{f}) + \lambda R_{\text{adv}, 0-1}(\hat{f}) > R_{0-1}(f^*) + \lambda R_{\text{adv}, 0-1}(f^*).
\]
This shows that $\hat{f}$ can't be a minimizer of the joint objective and minimizer of joint objective should be a Bayes optimal classifier.

\paragraph{Logistic Loss.} We now consider the logistic loss and show that any minimizer of $R(f) + \lambda R_{\text{adv}}(f)$ is a Bayes optimal classifier. We again prove the result by contradiction. Let $\xi = \log{\frac{1 + 2\gamma}{1 - 2 \gamma}}$. Suppose $\hat{f}$ is a minimizer of the joint objective and is not Bayes optimal.  Define set $X$ as
\[
X = \{\x : \hat{f}(\x)g(\x) < \xi\}.
\]
Note that, since $\hat{f}$ is not Bayes optimal, $X$ is a set with non-zero measure.
Construct a new classifier $\bar{f}$ as follows
\[
\bar{f}(\x) = \begin{cases}
\hat{f}(\x), \quad &\text{if } \x \not\in X\\
\hat{f}(\x) + \tau\left(\xi - \hat{f}(\x)g(\x)\right) g(\x), \quad &\text{otherwise}
\end{cases},
\]
where $\tau \in (0,1)$ is a constant.
We now show that $\bar{f}$ has a strictly lower joint risk than $\hat{f}$. This will then contradict our assumption that $\hat{f}$ is a minimizer of the joint objective.\\

Let $\ell_{\text{adv}}(f,g, \x)$ be the adversarial risk at point $\x$, computed w.r.t base classifier $g$
\[
\ell_{\text{adv}}(f,g, \x) =  \max_{\begin{subarray}{c} \|\bdelta\|\leq \epsilon \\  g(\x) = g(\x+\bdelta) \end{subarray}} \ell\left(f(\x+\bdelta), g(\x) \right) - \ell\left(f(\x), g(\x) \right).
\]
Define the conditional risk of $f$ at $\x$ as
\[
C(f, \x) = \mathbb{E}_y\left[\ell(f(\x), y)\Big| \x\right] + \lambda \ell_{\text{adv}}(f,g, \x).
\]
Note that $\mathbb{E}_{\x}\left[C(f, \x)\right]$ is equal to the joint risk $R(f)+\lambda R_{\text{adv}}(f)$. 
We now show that $C(\bar{f},\x)-C(\hat{f},\x) \leq 0, \forall \x$.
\paragraph{Case 1.} Let $\x \not\in X$. Then $\hat{f}(\x) = \bar{f}(\x)$. So we have
\begin{equation*}
    \begin{array}{lll}
         C(\bar{f}, \x) - C(\hat{f}, \x) &=& \lambda \left(\ell_{\text{adv}}(\bar{f},g, \x)-\ell_{\text{adv}}(\hat{f},g, \x)\right) \\
         & \leq &\displaystyle \lambda \left(\max_{\begin{subarray}{c} \|\bdelta\|\leq \epsilon \\  g(\x) = g(\x+\bdelta) \end{subarray}}\ell\left(\bar{f}(\x+\bdelta),g(\x)\right) - \ell\left(\hat{f}(\x+\bdelta),g(\x)\right)\right)\\
         &=& \displaystyle\lambda\max\left\lbrace 0,  \max_{\begin{subarray}{c} \|\bdelta\|\leq \epsilon, \x+\bdelta \in X \\  g(\x) = g(\x+\bdelta) \end{subarray}}\ell\left(\bar{f}(\x+\bdelta),g(\x)\right) - \ell\left(\hat{f}(\x+\bdelta),g(\x)\right) \right\rbrace\\
         &=&0,
    \end{array}
\end{equation*}
where the last equality follows from the observation that $g(\x)\bar{f}(\x+\bdelta) \geq g(\x)\hat{f}(\x+\bdelta)$ and the logistic function $\ell(z) = \log(1+e^{-z})$ is a monotonically decreasing function. 

\paragraph{Case 2.} Let $\x \in X$. Then $\hat{f}(\x) \neq \bar{f}(\x)$. 
Now, consider the difference $C(\bar{f}, \x) - C(\hat{f}, \x)$:
\[
C(\bar{f}, \x) - C(\hat{f}, \x) = \underbrace{\mathbb{E}_y\left[\ell(\bar{f}(\x), y) - \ell(\hat{f}(\x), y)\Big| \x\right]}_{T_1} + \lambda\underbrace{\left(\ell_{\text{adv}}(\bar{f},g, \x)-\ell_{\text{adv}}(\hat{f},g, \x)\right)}_{T_2}.
\]
We show that both $T_1,T_2$ are non-positive. 
Using the monotonicity property of logistic loss in Lemma~\ref{lem:logistic_monotone}, it is easy to verify that $T_1 < 0$. 
We now bound $T_2$. First, observe that based on our construction of $\bar{f}(\x)$ and our definition of set $X$, we have 
$$\inf_{\x \not \in X}\bar{f}(\x)g(\x) \geq \sup_{\x \in X} \bar{f}(\x)g(\x), \quad \inf_{\x \not \in X}\hat{f}(\x)g(\x) \geq \sup_{\x \in X} \hat{f}(\x)g(\x).$$
Since the logistic loss $\ell(z) = \log(1+e^{-z})$ is monotonically decreasing in $z$, this shows that both the inner maxima in $T_2$ are achieved  at $\bdelta$'s such that $\x +\bdelta \in X$. 
Using this observation, $T_2$ can be rewritten as
\[
\lambda\left(\max_{\begin{subarray}{c} \|\bdelta\|\leq \epsilon, \x+\bdelta \in X \\  g(\x) = g(\x+\bdelta) \end{subarray}}\ell\left(\bar{f}(\x+\bdelta),g(\x) \right) - \ell\left( \bar{f}(\x),g(\x) \right)\right) - \lambda\left(\max_{\begin{subarray}{c} \|\bdelta\|\leq \epsilon, \x+\bdelta \in X\\  g(\x) = g(\x+\bdelta) \end{subarray}}\ell(\hat{f}(\x+\bdelta),g(\x)) - \ell(\hat{f}(\x) , g(\x))\right).
\]
The above expression can be rewritten as
\[
\lambda \left(\max_{\begin{subarray}{c} \|\bdelta\|\leq \epsilon, \x+\bdelta \in X \\  g(\x) = g(\x+\bdelta) \end{subarray}}\ell\left(\bar{f}(\x+\bdelta),g(\x) \right) - \max_{\begin{subarray}{c} \|\bdelta\|\leq \epsilon, \x+\bdelta \in X\\  g(\x) = g(\x+\bdelta) \end{subarray}}\ell(\hat{f}(\x+\bdelta),g(\x))\right) - \lambda\left(\ell\left( \bar{f}(\x),g(\x) \right)  - \ell(\hat{f}(\x) , g(\x))\right).
\]
Note that $\ell\left(\bar{f}(\x+\bdelta),g(\x) \right)$ in the above expression can equivalently be written as $\ell\left((1-\tau)\hat{f}(\x+\bdelta) + \tau \xi g(\x),g(\x) \right)$. This shows that both  $\ell\left(\bar{f}(\x+\bdelta),g(\x) \right)$ and $\ell(\hat{f}(\x+\bdelta),g(\x))$ in the above expression are monotonically decreasing in $g(\x)\hat{f}(\x+\bdelta)$ and as a result the maximum of both the inner objectives is achieved at a $\bdelta$ which minimizes $g(\x)\hat{f}(\x+\bdelta)$. Let $\bdelta_{\x}$ be the point at which the maxima is achieved. Then the above expression can be written as
\[
T_2 = \lambda \left(\ell\left(\bar{f}(\x+\bdelta_{\x}),g(\x) \right) - \ell(\hat{f}(\x+\bdelta_{\x}),g(\x))\right) - \lambda\left(\ell\left( \bar{f}(\x),g(\x) \right)  - \ell(\hat{f}(\x) , g(\x))\right).
\]          
From Lemma~\ref{lem:logistic_monotone2} we know that $\ell\left( \bar{f}(\x),g(\x) \right)  - \ell(\hat{f}(\x) , g(\x))$ is an increasing function in $\hat{f}(\x)g(\x)$. Since $\hat{f}(\x)g(\x) \geq \hat{f}(\x+\bdelta_{\x}),g(\x+\bdelta_{\x})$, we have
\[
T_2 \leq 0.
\]
Combining the bounds for $T_1$ and $T_2$ we obtain $C(\bar{f}, \x) - C(\hat{f}, \x) < 0$, for any $\x \in X$. This shows that $\bar{f}(\x)$ has a strictly lower joint risk than $\hat{f}$. So $\hat{f}$ can't be a minimizer of the joint risk. This finishes the proof of Theorem~\ref{lem:minimizers}.
\section{Proof of Theorem~\ref{lem:constraint_removal}}
The proof follows from the proof of Theorem~\ref{lem:equiv}, because under the margin condition $H_{\text{adv}, 0-1}(f)$ is equivalent to $G_{\text{adv},0-1}(f)$ when the label $y$ is a deterministic function of $\x$.
\section{Proof of Theorem~\ref{lem:equiv}}
\paragraph{$\mathbf{0/1}$ loss.} We first prove the Theorem for $0/1$ loss. 
We use a similar proof strategy as Theorem~\ref{lem:minimizers} and prove the result by contradiction. Let $\eta(\x)$ be a Bayes decision rule which satisfies the margin condition. Let $f^*$ be a Bayes optimal classifier such that $\text{sign}(f^*(\x)) = \eta(\x)$ a.e. Suppose $\hat{f}$ is a minimizer of the joint objective. Let $\text{sign}(\hat{f}(\x))$ disagree with  $\text{sign}(f^*(\x))$ over a set $X$ of non-zero measure. From the proof of Theorem~\ref{lem:minimizers} we know that $R_{0-1}(\hat{f}) - R_{0-1}(f^*) > 0$.

We now show that $\hat{f}$ has a larger adversarial risk than $f^*$. From the definition of $G_{\text{adv}, 0-1}(f^*)$ we have
\begin{equation*}
    \begin{array}{lll}
         G_{\text{adv}, 0-1}(f^*) &=& \displaystyle \mathbb{E}_{(\x, y)}\left[ \max_{\|\bdelta\|\leq \epsilon } \ell_{0-1}\left(f^*(\x+\bdelta), y \right) - \ell_{0-1}\left(f^*(\x), y \right) \right]
    \end{array}.
\end{equation*}
From margin condition in Equation~\eqref{eqn:margin} we know that $\forall \x, \|\bdelta\| \leq \epsilon, \text{sign}(f^*(\x+\bdelta)) = \eta(\x+\bdelta) = \eta(\x)$. So $G_{\text{adv}, 0-1}(f^*) = 0$.

Since $G_{\text{adv}, 0-1}$ of any classifier is always non-negative, this shows that $G_{\text{adv}, 0-1}(\hat{f}) \geq G_{\text{adv}, 0-1}(f^*)$.  Combining this with the above result on classification risk we get
\[
R_{0-1}(\hat{f}) + \lambda G_{\text{adv}, 0-1}(\hat{f}) > R_{0-1}(f^*) + \lambda G_{\text{adv}, 0-1}(f^*).
\]
This shows that $\hat{f}$ can't be a minimizer of the joint objective. This shows that any minimizer of Equation~\eqref{eqn:objective_standard} is a Bayes optimal classifier.


\paragraph{Logistic loss.}  To prove the Theorem for logistic loss, we heavily rely on some of the intermediate results we proved for Theorem~\ref{lem:minimizers}. Let $\xi = \log{\frac{1 + 2\gamma}{1 - 2 \gamma}}$. Suppose $\hat{f}$ is a minimizer of the joint objective and is not Bayes optimal.  Define set $X$ as
\[
X = \{\x : \hat{f}(\x)\eta(\x) < \xi\}.
\]
Note that $X$ is a set with non-zero measure.
Construct $\bar{f}$ as follows
\[
\bar{f}(\x) = \begin{cases}
\hat{f}(\x), \quad &\text{if } \x \not\in X\\
\hat{f}(\x) + \left(\xi - \hat{f}(\x)\eta(\x)\right)\tau \eta(\x), \quad &\text{otherwise}
\end{cases},
\]
where $\tau \in (0,1)$ is a constant.
We now show that $\bar{f}$ has a strictly lower joint risk than $\hat{f}$.\\

Define the conditional risk of $f$ at $\x$ as
\[
C(f, \x) = \mathbb{E}_{y}\left[\ell(f(\x), y)\Big| \x\right] + \lambda \mathbb{E}_{y}\left[ \max_{\|\bdelta\|\leq \epsilon } \ell\left(f(\x+\bdelta), y \right) - \ell\left(f(\x), y\right)\Big| \x\right].
\]
We consider two cases, $\x \in X$ and $\x \not\in X$, and show that in both the cases $\bar{f}$ has a lower conditional risk than $\hat{f}$.
\paragraph{Case 1.} Let $\x \not\in X$. Then $\hat{f}(\x) = \bar{f}(\x)$. So we have
\begin{equation*}
    \begin{array}{lll}
         C(\bar{f}, \x) - C(\hat{f}, \x) &=& \displaystyle\lambda \mathbb{E}\left[\max_{\|\bdelta\|\leq \epsilon}\ell\left(\bar{f}(\x+\bdelta),y\right) - \max_{\|\bdelta\|\leq \epsilon}\ell\left(\hat{f}(\x+\bdelta),y\right)\Big| \x\right]  \\
         & =& \lambda P(y = \eta(\x)|\x) \underbrace{\left[\max_{\|\bdelta\|\leq \epsilon}\ell\left(\bar{f}(\x+\bdelta),\eta(\x)\right) -\max_{\|\bdelta\|\leq \epsilon}\ell\left(\hat{f}(\x+\bdelta),\eta(\x)\right)\right]}_{T_1}\\
         &&+\lambda P(y = -\eta(\x)|\x) \underbrace{\left[\max_{\|\bdelta\|\leq \epsilon}\ell\left(\bar{f}(\x+\bdelta),-\eta(\x)\right) -\max_{\|\bdelta\|\leq \epsilon}\ell\left(\hat{f}(\x+\bdelta),-\eta(\x)\right)\right]}_{T_2}
    \end{array}
\end{equation*}
Using the margin condition on $\eta(\x)$, and using the same technique as in proof of Case 1 of Theorem~\ref{lem:minimizers}, we can show that $T_1 \leq 0$. Since both the inner maxima in $T_2$  are achieved at $\x+\bdelta \not\in X$, it is easy to verify that $T_2 = 0$. This shows that $\forall \x \not\in X, C(\bar{f}, \x) - C(\hat{f}, \x) \leq 0$.
\paragraph{Case 2.} Let $\x \in X$. Let $\ell_{\text{adv}}(f,\x, y)$ be the adversarial risk at point $(\x,y)$
\[
\ell_{\text{adv}}(f, \x, y) =  \max_{\|\bdelta\|\leq \epsilon } \ell\left(f(\x+\bdelta),y \right) - \ell\left(f(\x), y \right).
\] 
We have
\[
C(\bar{f}, \x) - C(\hat{f}, \x) = \underbrace{\mathbb{E}\left[\ell(\bar{f}(\x), y) - \ell(\hat{f}(\x), y)\Big| \x\right]}_{T_3}
 + \lambda\underbrace{\mathbb{E}_y\left[\ell_{\text{adv}}(\bar{f}, \x, y)-\ell_{\text{adv}}(\hat{f}, \x, y)\right]}_{T_4}.
\]
From the proof of Case 2 of Theorem~\ref{lem:minimizers} we know that $T_3 < 0$. We now show that $T_4 \leq 0$. Let $p_{\x} = P(y = \eta(\x)|\x)$.
$T_4$ can be decomposed as follows
\[
p_{\x}\underbrace{\left(\ell_{\text{adv}}(\bar{f}, \x, \eta(\x))-\ell_{\text{adv}}(\hat{f}, \x, \eta(\x))\right)}_{T_5} + 
(1-p_{\x})\underbrace{\left(\ell_{\text{adv}}(\bar{f}, \x, -\eta(\x))-\ell_{\text{adv}}(\hat{f}, \x, -\eta(\x))\right)}_{T_6}.
\]
Following the proof of Case 2 of Theorem 1 and using the margin condition we can show that $T_5 \leq 0$. We now show that $T_6 \leq 0$. First observe that both the suprema in $T_6$ either occur at the same point. Suppose both the suprema in $T_6$ occur outside $X$. Then $T_6$ is given by
\[
T_6 = \ell(\hat{f}(\x),-\eta(\x)) - \ell(\bar{f}(\x),-\eta(\x)) \leq 0.
\]
Suppose both the suprema occur inside $X$. Then using the observation that $\ell(\bar{f}(\x), -\eta(\x)) - \ell(\hat{f}(\x), -\eta(\x))$ is a decreasing function of $\hat{f}(\x)\eta(\x)$ (see Lemma~\ref{lem:logistic_monotone2}), we get $T_6 \leq 0$.
\section{Proof of Theorem~\ref{thm:nomargin}}
For any Bayes decision rule $\eta(\x)$, let $X_{\eta}$ be the set of points which violate the margin condition
\[
X_{\eta} = \left\lbrace \x: \exists \tilde{\x},  \|\tilde{\x} - \x\| \leq \epsilon \text{ and } \eta(\tilde{\x}) \neq \eta(\x)\right\rbrace.
\]
Since no Bayes decision rule satisfies the margin condition, we have $\text{Pr}(\x \in X_{\eta}) > 0, \forall \eta$. Let $p = \inf_{\eta} \text{Pr}(\x \in X_{\eta})$. 

We first consider the joint risk $R_{0-1}(f)+\lambda G_{\text{adv}, 0-1}(f)$. To prove the Theorem we show that there exist classifiers which obtain strictly lower joint risk than any Bayes optimal classifier.
Let $f^*:\mathbb{R}^d \rightarrow \mathbb{R}$ be any Bayes optimal classifier, with the corresponding Bayes decision rule $\eta(\x) = \text{sign}(f^*(\x))$. We first obtain a lower bound on  $G_{\text{adv}, 0-1}(f^*)$. Consider the following
\begin{equation*}
    \begin{array}{lll}
G_{\text{adv}, 0-1}(f^*) &\geq& \displaystyle \text{Pr}(\x \in X_{\eta})\times \mathbb{E}_{(\x,y)}\left[\sup_{\begin{subarray}{c} \|\bdelta\|\leq \epsilon \end{subarray}} \ell_{0-1}\left(f^*(\x+\bdelta), y \right) - \ell_{0-1}\left(f^*(\x), y \right)\Big| \x \in X_{\eta}\right]\\
&\geq&  \displaystyle  \text{Pr}(\x \in X_{\eta})\times\mathbb{E}_{\x}\left[P(y=\eta(\x)|\x)\left(\sup_{\begin{subarray}{c} \|\bdelta\|\leq \epsilon \end{subarray}} \ell_{0-1}\left(f^*(\x+\bdelta), \eta(\x) \right) - \ell_{0-1}\left(f^*(\x), \eta(\x) \right)\right)\Big|\x \in X_{\eta}\right]\\
&\geq& \displaystyle \frac{\text{Pr}(\x \in X_{\eta})}{2}\mathbb{E}_{\x}\left[\sup_{\begin{subarray}{c} \|\bdelta\|\leq \epsilon \end{subarray}} \ell_{0-1}\left(f^*(\x+\bdelta), \eta(\x) \right) - \ell_{0-1}\left(f^*(\x), \eta(\x) \right)\Big|\x \in X_{\eta}\right]\\
&\geq& \displaystyle\frac{\text{Pr}(\x \in X_{\eta})}{2}\\
&\geq& \displaystyle\frac{p}{2}
    \end{array}
\end{equation*}
where the third inequality follows from the fact that $P(y=\eta(\x)|\x) \geq \frac{1}{2}$ and the fourth inequality follows from the observation that any $\x \in X_{\eta}$ violates the margin condition. 
This gives us the following lower bound on the joint risk of $f^*$
\begin{equation}
\label{eqn:nomargin_intd1}
    R_{0-1}(f^*)+\lambda G_{\text{adv}, 0-1}(f^*) \geq \frac{\lambda p}{2}.
\end{equation}
Now consider the ``constant'' classifier $f_{-1}$ which assigns all the points to the negative class. This classifier has $0$ adversarial risk. So its joint risk can be upper bounded as follows
\begin{equation}
\label{eqn:nomargin_intd2}
R_{0-1}(f_{-1})+\lambda G_{\text{adv}, 0-1}(f_{-1}) \leq 1.
\end{equation}
Equations~\eqref{eqn:nomargin_intd1},~\eqref{eqn:nomargin_intd2} show that $\forall \lambda > \frac{2}{p}$, there exist classifiers with strictly lower joint risk than any Bayes optimal classifier.  Using the same argument we can show that similar results hold for the other joint risk $R_{0-1}(f)+\lambda H_{\text{adv}, 0-1}(f)$. 
\section{Proofs of Section~\ref{sec:smaller_spaces}}
Here we present the proofs of Section~\ref{sec:smaller_spaces}. To begin with, we first present a result which characterizes the standard and adversarial risk for the mixture model.

\begin{theorem}\label{thm:mixture}
Suppose the perturbations are measured w.r.t $L_{\infty}$ norm. Let $\w \in \real^p$ be a linear separator and moreover suppose the base classifier $g(\x)$ is the Bayes optimal decision rule. Then, for any linear classifier $f_{\w}(\x) = \w^T \x$, we have that
\begin{enumerate}
\item $ R_{0-1}(f_{\w}) = \Phi\paren{ - \frac{\w^T \tparam}{\sigma \norm{ \w}{2}}},$
\item $ G_{\text{adv}, 0-1}(f_{\w}) = \Phi \paren{ \frac{\norm{\w}{1}\epsilon - \w^T \tparam}{\sigma \norm{\w}{2}}}, $
\item $ R_{\text{adv},0-1}(f_{\w}) \leq \Phi \paren{ \frac{\norm{\w - \tparam}{1}\epsilon - (\w - \tparam)^T \tparam}{\sigma \norm{\w - \tparam}{2}}}$,
\end{enumerate}
where $\Phi(\cdot)$ is the CDF of the standard normal distribution.
\end{theorem}
\begin{proof}
To see the first part, we begin by observing that $\w^T \x$ is a univariate normal random variable when conditioned
on the label $y$, one can derive the 0-1 error for the classifier in closed form. In particular, 
\[ R_{0-1}(f_{\w})  = 1 - \half \Phi \paren{\frac{\w^T \tparam}{\sigma \norm{\w}{2}}} - \half \Phi \paren{\frac{\w^T \tparam}{\sigma \norm{\w}{2}}} = 1 - \Phi \paren{\frac{\w^T \tparam}{\sigma \norm{\w}{2}}} =  \Phi \paren{\frac{-\w^T \tparam}{\sigma \norm{\w}{2}}} \]

Following the existing definition of adversarial risk, we see that
\[ G_{adv, 0-1}(f) = \Exp \sqprn{\max \limits_{\bdelta: \norm{\bdelta}{\infty} \leq \epsilon} \ell_{0-1}(f(\x+\bdelta),y)}  \]
We consider the case of $y=1$. We know that $\x|y=1 \sim \calN(\tparam,\sigma^2\calI_d)$. So, $\x = \tparam + \z $, where $\z \sim \calN(0,\sigma^2\calI_d)$. Now, for any $\z$, we incur a loss of 1, whenever there exists a $\bdelta$ such that $\norm{\bdelta}{\infty} \leq \epsilon$ and,
\[ \w^T(\x + \bdelta) = \w^T(\tparam) + \w^T(\z) + \w^T \bdelta \leq 0, \]
As long as $\z$ is such that, $\w^T \z \leq \norm{\w}{1}\epsilon - \w^T \tparam$, we will always incur a penalty. Now, $\w^T\z \sim \calN(0,\sigma^2 \norm{\w}{2}^2)$, therefore, 
$Pr(\w^T \z \leq \norm{\w}{1}\epsilon - \w^T\tparam) = \Phi(\frac{\norm{\w}{1}\epsilon - \w^T\tparam}{\norm{\w}{2}\sigma})$. Symmetric argument holds for $y=-1$. Hence, we get that,
\[ G_{adv, 0-1}(f_{\w}) = \Phi \paren{\frac{\norm{\w}{1}\epsilon - \w^T\tparam}{\norm{\w}{2}\sigma}} \]
Now to prove the third claim, we have that
\bit 
\item Suppose $y=1$, then $\x = \tparam + z$ where $\z \sim \calN(0,\sigma^2 \calI_d)$. Suppose $\tparam^T \x > 0$.
\item Then, for a given $\z$, we will incur a penalty if $\z$ satisfies the following constraints:
\bit 
\item We have that $\w^T \x = \w^T \tparam + \tparam^T \z > 0$.
\item There exists a $\bdelta\  \st \norm{\bdelta}{\infty} \leq \epsilon$, such that,
\[ \tparam^T(\x + \bdelta) > 0~~\text{and}~~\w^T(\x+\bdelta) < 0 \]
\item Note that whenever the above event happens, the following also happens:
\[ (\w - \tparam)^T(\x + \bdelta) = (\w - \tparam)^T(\z) + (\w - \tparam)^T(\tparam)  + (\w - \tparam)^T\bdelta < 0 \]
Now, for a given $\z$, $(\w - \tparam)^T(\z) \sim \calN(0,\norm{\w -\tparam }{2}^2)\sigma^2$. Also, as long as $\z$ is such that $(\w - \tparam)^T(\z) \leq \norm{\w - \tparam}{1} \epsilon - (\w - \tparam)^T \tparam$, we will incur a penalty. This event happens with probability, 
\[ \Phi \paren{\frac{\norm{\w - \tparam}{1} \epsilon - (\w - \tparam)^T \tparam}{\sigma \norm{\w - \tparam}{2}}} \]
This establishes the upper bound.
\eit 
\eit 
\end{proof}


\subsection{Proof of Theorem~\ref{cor:sparse}}
We use the same notation as in the proof of Theorem~\ref{thm:mixture}. Let $R^* = R_{0-1}(f_{\tparam})$. Using Theorem~\ref{thm:mixture}, we can write the excess 0-1 risk of $\w$ as:
\[ R_{0-1}(f_{\w}) - R^* = \Phi \paren{- \frac{\norm{\tparam}{2}^2}{\sigma \paren{\sqrt{\norm{\tparam}{2}^2 +1}}}} - \Phi \paren{ - \frac{\norm{\tparam}{2}}{\sigma}} \]
\[ R_{0-1}(f_{\w}) - R^* = \Phi \paren{- \frac{\norm{\tparam}{2}}{\sigma \paren{\sqrt{1 + \frac{1}{\norm{\tparam}{2}^2}}}}} - \Phi \paren{ - \frac{\norm{\tparam}{2}}{\sigma}} \]
Next, we lower bound the adversarial risk. Suppose that $y=1$, then we have that $\x = \tparam + \z_S + \z_{S^c}$. Similarly, let $\w = \w_S+\w_{S^c}$. In our case, $\w_S = \tparam$ and $\w_{S^c} = \balpha = [\frac{\pm 1 }{\sqrt{d-k}},\frac{\pm 1}{\sqrt{d-k}},\ldots,\frac{\pm 1}{\sqrt{d-k}}]^T$. Then, we have that $\w^T \x = \tparam^T\tparam + \tparam^T \z_S + \balpha^T \z_{S^c}$. 
\bit 
\item Consider the Event $\balpha^T \z_{S^c} > -\tparam^T \tparam - \tparam^T \z_S$ This is the event that $\w,\tparam$ agree before perturbation.
\item Consider the Event B,
\[ \balpha^T \z_{S^c} < \norm{\balpha}{1}\epsilon - \tparam^T \tparam - \tparam^T\z_S \]
This is the event that there exists a perturbation restricted to the subspace $S^c$ such that, $\w^T(\x+\bdelta) < 0$. Note that since the perturbation is restricted to $S^c$, $\tparam$'s prediction doesn't change.
\item Now for the probability that both events happen:
\bit 
\item Observe that $A = (\balpha^T\z_{S^c} + \tparam^T \z) \sim \calN(0,\sigma^2(\norm{\balpha}{2}^2 + \norm{\tparam}{2}^2) )$. 
\item So, the probability of both events happening is that the random variable $ - \tparam^T \tparam \leq A \leq  \norm{\balpha}{1} \epsilon - \tparam^T \tparam $
\[ \Phi \paren{\frac{\norm{\balpha}{1}\epsilon - \tparam^T \tparam}{\sigma \sqrt{(\norm{\balpha}{2}^2 + \norm{\tparam}{2}^2)}}} - \Phi \paren{ \frac{-\tparam^T \tparam}{\sigma \sqrt{(\norm{\balpha}{2}^2 + \norm{\tparam}{2}^2)}}} \]
\item Now, for $\epsilon = 2\norm{\tparam}{2}^2 / \sqrt{d-k}$, we get that the probability that both events happens is:
\[ \Phi \paren{\frac{\tparam^T \tparam}{\sigma \sqrt{(\norm{\balpha}{2}^2 + \norm{\tparam}{2}^2)}}} - \Phi \paren{ \frac{-\tparam^T \tparam}{\sigma \sqrt{(\norm{\balpha}{2}^2 + \norm{\tparam}{2}^2)}}} = 2 \Phi \paren{\frac{\tparam^T \tparam}{\sigma \sqrt{(\norm{\balpha}{2}^2 + \norm{\tparam}{2}^2)}}} - 1 \]
\item Now, for $\frac{\tparam^T \tparam}{\sigma \sqrt{(\norm{\balpha}{2}^2 + \norm{\tparam}{2}^2)}} = 2$, $R_{adv,0-1}(f) > 0.95$ 
\item Note that $\norm{\balpha}{2}^2 = 1$. Therefore for $\sigma=1$, we get that $\norm{\tparam}{2}^2 = 2 + 2*\sqrt{2}$. 
\item At this value, we have that excess 0-1 risk < 0.02, which completes the proof.
\eit 
\eit

\subsection{Proof of Theorem~\ref{lem:functionClass}}
We know that $\tparam = [1/\sqrt{d/2},1/\sqrt{d/2},\ldots,1/\sqrt{d/2}] \in \real^p$. When restricted to only top half co-ordinates, it is easy to see that $\w = [\underbrace{1/\sqrt{d/2},\ldots,1/\sqrt{d/2}}_{d/2},0,\ldots,0]$ is the optimizer of the standard risk. For this setting, from Theorem~\ref{thm:mixture}, we get that,
\[ R_{0-1}(f_{\tparam}) = \Phi( - \sqrt{2}) = 0.07 ,~~~R_{0-1(f_{\w})} = \Phi(-1) = 0.16 \]
Hence, we have that $R_{0-1}(f_{\w}) - R_{0-1}(f_\tparam) < 0.1$. Now, to get a lower bound on the adversarial risk of $\w$, consider the perturbations of the form $\bgamma = [\underbrace{-\epsilon,-\epsilon,\ldots,-\epsilon}_{d/2}, \epsilon ,\epsilon,\ldots,\epsilon]$. Note that for such a perturbation $\gamma$, we have that,
\[ \tparam^T \x = \tparam^T (\x + \bgamma)~~\text{and}~~\w^T(\x+\bgamma) = \w^T \x - \epsilon \sqrt{\frac{d}{2}} \]

Now, suppose $y=1$. Then, $\x = \tparam + \z$, where $\z \sim \calN(0,\calI_d)$. For this, we have that,
\[ \tparam^T(\x) = \tparam^T \tparam + \tparam^T\z = 2 + \underbrace{\w_{1:d/2}^T \z_{1:d/2}}_{A} +  \underbrace{\w_{1:d/2}^T \z_{d/2:d}}_{B}, \]
On the other hand, $\w^T \x = \w^T \tparam + \w^T\z = 1 + \underbrace{\w_{1:d/2}^T \z_{1:d/2}}_{A}$. Consider the event such that \[ \tparam^T \x > 0 ~ \& ~ \w^T \x > 0 ~ \& ~ \w^T(\x+\bgamma) < 0 . \] This is the event that $x$ is such that both of $\w$ and $\w^*$ agree before, but after adding the perturbation $\gamma$ the prediction of $\w$ changes. Following the form of $\bgamma$, this event can be rewritten as:
\[ \tparam^T \x > 0 ~ \& ~ \w^T \x > 0 ~ \& ~ \w^T(\x) < \epsilon \sqrt{d/2}   \]
Rewriting this event in terms of the random variables $A$ and $B$, we get the equivalent event,
\[ 2 + A + B > 0 ~~ \& ~~ A + 1 > 0 ~~ \& ~~ A + 1 < \epsilon \sqrt{d/2}, \]
where $A$ and $B$ are independent and zero-mean unit variance gaussians, \ie $A,B \sim \calN(0,1)$. We just need to lower bound the probability of this event. Consider the distribution of $A$ conditioned on $A+B > -2$, suppose its CDF is $F$, then the probability of the event above is $F(\epsilon \sqrt{d/2} - 1) - F(-1)$. Now, to derive an expression for $F$, 
\[ F(a) = P(A \leq a | (A + B) > -2 ) = \frac{P(A \leq a~ \& ~ ((A+B) > -2))}{P((A+B)>-2)}, \]
Using that $A + B \sim \calN(0,2)$, we get 
\[ F(a) =  (1 - \Phi(-\sqrt{2}))) \int_{-\infty}^a P(B > -2 - u) \phi(u) du \]
where $\Phi(\cdot)$ and $\phi(\cdot)$ are the cdf and pdf of standard normal. Hence, we get that for a suitable constant $\epsilon \geq C/\sqrt{d}$ the probability of this event is lower bounded by 0.95.

\subsection{Proof of Theorem~\ref{cor:surr}}
Suppose gradient descent is initialized at $\w^0$. Let $\w^t$ be the $t^{th}$ iterate of GD. Note that the gradients of the loss function are always in the span of the covariates $\x_i$. Hence, any iterate of gradient descent lies in $\w^0 + \text{span}(\{\x_i\}_{i = 1}^n)$. Let $S$ be the indices corresponding to the non-zero entries in $\tparam$. Since the covariates lie in a low dimensional subspace and are 0 outside the subspace, the co-ordinates of $\w^t$ satisfy the invariant, $$ \w^t_{S^c} = \w^0_{S^c}.$$ Moreover, since we initialized $\w^0$ using a random gaussian initialization with covariance $\frac{1}{\sqrt{d-k}} \calI_d$, we know that with high probability,
\[ \norm{\w^0_{S^c}}{1} = \sqrt{d-k}~~~\text{and}~~~\norm{\w^0_{S^c}}{2} = O(1) \]

Next, we lower bound the adversarial risk. Suppose we fix $y=1$, then we have that for any $\x = \tparam + \z_S$. Note that $\z_S^c = 0$. \\

We can rewrite the $\hat{\w}_{GD} = \w =  \w_S + \underbrace{\w_{S^c}}_{ = \balpha}$ where $\w_S$ is the component in the low dimensional mixture subspace, and $\balpha$ is the component in the complementary subspace $S^c$. As stated above, since the covariates lie in a low dimensional subspace, hence, the component in the complementary subspace doesn't get updated. Therefore, $\alpha = \w^0_{S^c}$. 

Now, for any $x$, we have that 
\begin{align*}
    \hat{\w}_{GD}^T x & = \w^T \x \\
    & = \w^T(\w^* + \z_S + \underbrace{\z_{S^c}}_{=0}) \\
    & = \w_S^T \tparam + \w_S^T\z_{S}
\end{align*} 

\bit 
\item Consider the event $\w_S^T \z_S > -\w_S^T \tparam$. This is the event that $\w,\tparam$ agree before perturbation.
\item Consider the event $B$ such that,
\[ \w_S^T \z_S < \norm{\balpha}{1} \epsilon - \w_S^T \tparam \]
This is the event that there exists a perturbation restricted to the subspace $S^c$ such that, the prediction of $\hat{w}_{GD}$ changes, \ie $\w^T(\x+\bdelta) < 0$. Note that since the perturbation is restricted to $S^c$, the prediction of $\tparam$ doesn't change.
\item Hence, both events happen if 
\[ -\w_S^T \tparam \leq \w_S^T \z_S \leq \norm{\balpha}{1} \epsilon - \w_S^T \tparam \]
\item To bound this probability, observe that $\w_S^T\z_S \sim \calN(0,\sigma^2 \norm{\w_S}{2}^2)$. Hence, 

\[ \Pr \paren{-\w_S^T \tparam \leq \w_S^T \z_S \leq \norm{\balpha}{1} \epsilon - \w_S^T \tparam} =  \Phi \paren{\frac{\norm{\balpha}{1} \epsilon - \w_S^T \tparam}{\sigma \norm{\w_S}{2}}}  - \Phi \paren{\frac{-\w_S^T\tparam}{\sigma \norm{\w_S}{2}}} \]
\item We know that from our initialization, $\norm{\balpha}{1} = \norm{w_{S^c}^0}{1} = \sqrt{d-k}$. Hence, for $\epsilon = 2 \w_S^T \tparam / (\sqrt{d-k})$, we get that both the events happen with probability,
\[ \Phi \paren{\frac{\w_S^T\tparam}{\sigma \norm{\w_S}{2}}}  - \Phi \paren{\frac{-\w_S^T\tparam}{\sigma \norm{\w_S}{2}}}  = 2 \Phi \paren{\frac{\w_S^T\tparam}{\sigma \norm{\w_S}{2}}}  - 1 \]
\item Since as gradient descent progresses, $\w_S \to \tparam$, this implies that for $\sigma=1$, and $\norm{w_S}{2} = 2$, we have that $R_{adv,0-1}(\w_S) > 0.95$ for a very small $\epsilon$ such that $\epsilon = \frac{C}{\sqrt{d-k}}$, where $C > 0$ is a small constant.
\eit

Plugging this into Theorem~\ref{thm:mixture}, we recover the result.

\section{Proof of Theorem~\ref{lem:reg_new}}
\paragraph{1.} First note that $f(\x+\bdelta)$ can be written as
\[
f(\x+\bdelta) = f(\x) + \int_{t = 0}^1 \nabla f(\x+t\bdelta)^T\bdelta dt.
\]
Rearranging the terms gives us:
\[
|f(\x+\bdelta) - f(\x)| \leq \left| \int_{t = 0}^1 \nabla f(\x+t\bdelta)^T\bdelta dt\right| \leq \epsilon \sup_{\|\bdelta\|\leq \epsilon}\|\nabla f(\x+\bdelta)\|_*.
\]
Let $u(\x) = \epsilon \sup_{\|\bdelta\|\leq \epsilon}\|\nabla f(\x+\bdelta)\|_*$. Since the loss $\ell$ is $1$-Lipschitz, we can upper bound $\ell(f(\x+\bdelta), y)$ as
\[
\ell(f(\x+\bdelta), g(\x)) - \ell(f(\x), g(\x)) \leq |f(\x+\bdelta) - f(\x)| \leq \epsilon \sup_{\|\bdelta\|\leq \epsilon}\|\nabla f(\x+\bdelta)\|_*.
\]
So we have the following upper bound for the objective in Equation~\eqref{eqn:objective_mod}
\begin{equation}
\label{eqn:reg_new_eq1}
R(f)+\lambda R_{\text{adv}}(f) \leq R(f) + \epsilon\lambda \mathbb{E}\left[ \sup_{\|\bdelta\|\leq \epsilon}\|\nabla f(\x+\bdelta)\|_* \right].
\end{equation}
\paragraph{2.} We now get a different upper bound for $|\ell(f(\x+\bdelta), g(\x+\bdelta)) - \ell(f(\x), g(\x))|$ in terms of $\|f-g\|_{\infty}$. Since $\ell$ is $1$-Lipschitz we have
\[
|\ell(f(\x+\bdelta), g(\x+\bdelta)) - \ell(f(\x), g(\x))| \leq |f(\x+\bdelta)g(\x+\bdelta)  - f(\x)g(\x)|.
\]
Note that $|f(\x)g(\x)|$ can be upper bounded by $|f(\x) - g(\x)|$. This gives us the following bound
\[
|\ell(f(\x+\bdelta), g(\x+\bdelta)) - \ell(f(\x), g(\x))| \leq |f(\x) - g(\x)| + |f(\x+\bdelta) - g(\x+\bdelta)|
\]
Substituting this in the definition of $R_{\text{adv}}(f)$ gives us the following  upper bound for the objective in Equation~\eqref{eqn:objective_mod}
\begin{equation}
\label{eqn:reg_new_eq2}
R(f)+\lambda R_{\text{adv}}(f) \leq R(f) + 2\lambda\|f-g\|_{\infty}.
\end{equation}
Combining Equations~\eqref{eqn:reg_new_eq1},~\eqref{eqn:reg_new_eq2} gives us the required result.
\section{Proof of Theorem~\ref{lem:reg_old}}
The proof of part (a) and upper bound of part (b) of the Theorem follow from the proof of Theorem~\ref{lem:reg_new}. Here, we focus on proving the lower bound of part (b).
The adversarial risk used in Equation~\eqref{eqn:objective_standard} can be rewritten as
\[
G_{\text{adv}}(f_{\w}) = \mathbb{E}\left[ \sup_{\begin{subarray}{c} \|\bdelta\|\leq \epsilon \end{subarray}} \ell\left(\w^T(\x+\bdelta), y \right) - \ell\left(\w^T\x, y \right) \right].
\]
Since $\ell\left(\w^T(\x+\bdelta), y \right)$ is maximized at a point where $y\w^T(\x+\bdelta)$ is minimized, we get the following expression for  $G_{\text{adv}}(f_{\w})$
\[
G_{\text{adv}}(f_{\w}) = \mathbb{E}\left[ \ell\left(\w^T\x - y\epsilon\|\w\|_*, y \right) - \ell\left(\w^T\x, y \right) \right].
\]
We now obtain a lower bound for $G_{\text{adv}}(f_{\w})$
\begin{equation}
\begin{array}{lll}
G_{\text{adv}}(f_{\w}) &=& P(y\w^T\x \leq 0) \times \mathbb{E}\left[ \ell\left(\w^T\x - y\epsilon\|\w\|_*, y \right) - \ell\left(\w^T\x, y \right) \Big| y\w^T\x \leq 0\right] \vspace{0.1in}\\
&& + P(y\w^T\x > 0) \times \mathbb{E}\left[ \ell\left(\w^T\x - y\epsilon\|\w\|_*, y \right) - \ell\left(\w^T\x, y \right) \Big| y\w^T\x > 0\right] \vspace{0.1in}\\
&\geq& P(y\w^T\x \leq 0) \times \mathbb{E}\left[ \ell\left(\w^T\x - y\epsilon\|\w\|_*, y \right) - \ell\left(\w^T\x, y \right) \Big| y\w^T\x \leq 0\right].
\end{array}
\end{equation}
Consider the logistic loss $\ell(z) = \log{1+e^{-z}}$. For $z < 0$, the absolute value of derivative of logistic loss is greater than $\frac{1}{2}$. This shows that for $(\x,y)$ such that $y\w^T\x \leq 0$, we  have
\[
\ell\left(\w^T\x - y\epsilon\|\w\|_*, y \right) - \ell\left(\w^T\x, y \right) \geq \frac{1}{2}\epsilon\|\w\|_*.
\]
This completes the proof of the Theorem.
Substituting this in the above lower bound for the adversarial risk $G_{\text{adv}}(f_{\w})$, we get
\[
G_{\text{adv}}(f_{\w}) \geq \frac{1}{2}\epsilon R_{0-1}(f_{\w})\|\w\|_*.
\]
\section{Choice of $\lambda$}
In this section we highlight the importance of choosing an appropriate hyper-parameter $\lambda$ in Equation~\eqref{eqn:objective_standard}. Most of the existing works~\citep{madry2017towards, zico} on adversarial training always choose $\lambda = 1$ for minimization of \eqref{eqn:objective_standard}. However, we note that $\lambda = 1$ may not always be the optimal choice. Indeed $\lambda$ captures the tradeoff between two distinct objects: the classification risk, and the \emph{excess risk} due to adversarial perturbations, and it is thus quite natural to expect the optimal tradeoff to occur at values other than $\lambda = 1$.
Figure~\ref{fig:cifar_pgd_lambda} shows the behavior of various risks of models obtained by minimizing objective~\eqref{eqn:objective_standard} on CIFAR10, for various choices of $\lambda$. We use VGG11 network with reduced capacity, where we reduce the number of units in each layer to $1/4^{th}$. It can be seen that as $\lambda$ increases adversarial risk goes down, but classification risk goes up, and an optimal choice of $\lambda$ should be based on the metric one cares about.
If one cares about the joint risk then $\lambda = 2$ has slightly better performance than $\lambda = 1$. 
\begin{figure}
\centering
\includegraphics[width=0.33\textwidth]{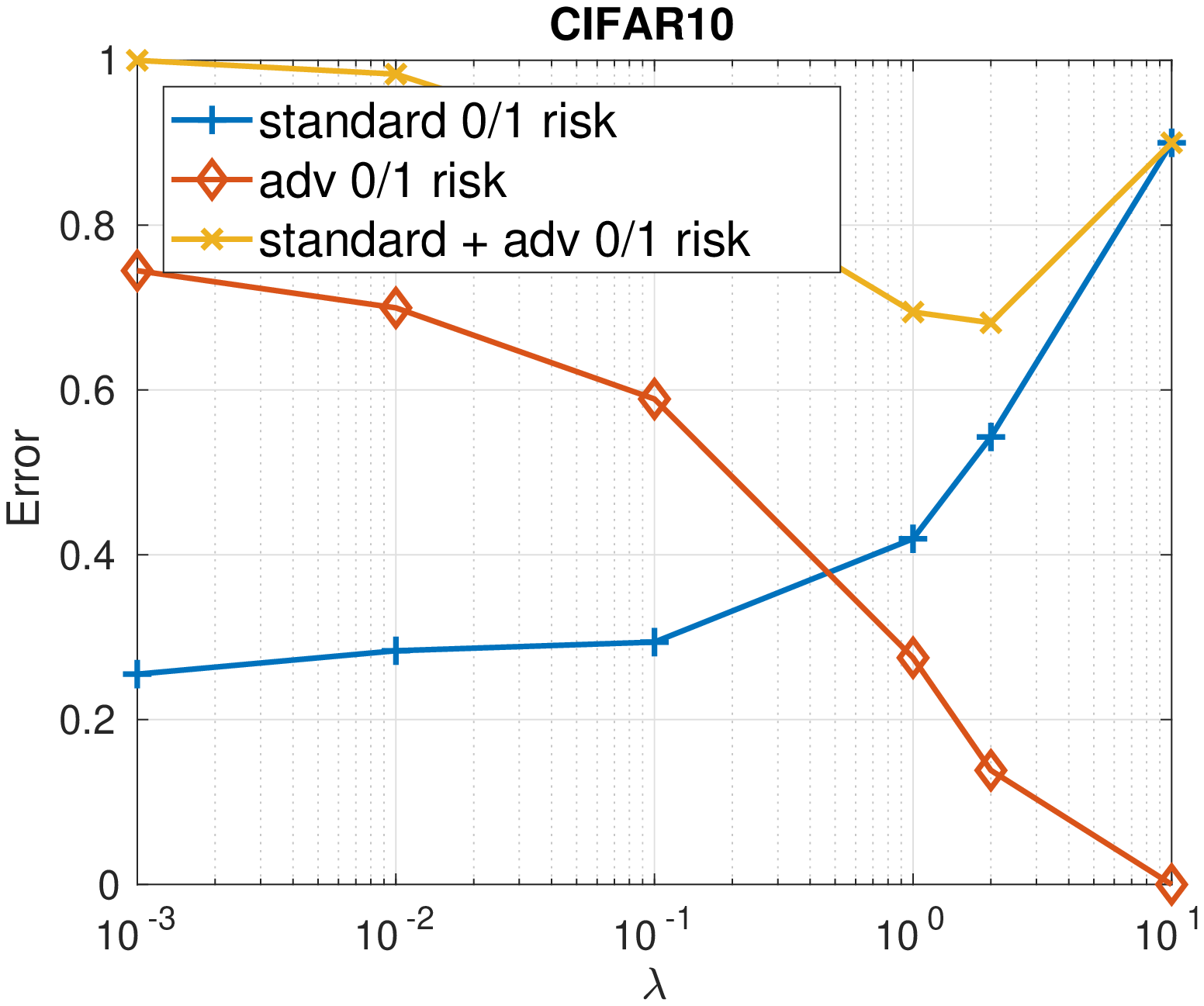}
\vspace{-0.12in}
\caption{ Figure shows the behavior of various risks obtained by adversarial training with $\epsilon = 0.03$, with varying $\lambda$. The adversarial perturbations in both the experiments are measured w.r.t $L_{\infty}$ norm.}
\label{fig:cifar_pgd_lambda}
\end{figure}

\section{Experimental Settings}
In all our experiments we use the following network architectures:
\paragraph{MNIST.} For all our experiments on MNIST, we use 1 hidden layer neural network with ReLU activations. To control the capacity of the network we vary the number of hidden units.
\paragraph{CIFAR10.} For all our experiments on CIFAR10, we use VGG11 network. To control the capacity of the network we scale the number of units in each layer. By a capacity scale of $\alpha$, we mean that we use $\alpha$ times the number of units in each layer of original VGG network.

\paragraph{PGD Training.}  In all our experiments we measure adversarial perturbations w.r.t $L_{\infty}$ norm and use projected gradient descent as our adversary. For PGD training on  MNIST, we optimize the inner maximization problem for $50$ iterations with step size $0.01$. For PGD training on CIFAR10, we optimize the inner maximization problem for $15$ iterations with step size $0.005$. The outer minimization is run for $40$ epochs for MNIST and $50$ epochs for CIFAR10 and we use SGD+momentum with learning rate $0.01$ and batch size $128$.

\paragraph{Computation of adversarial risk.} We use adversarial examples generated by PGD to compute the adversarial risk of a classifier. The hyper-parameter settings are the same as the one used for PGD training. 


\clearpage

\end{document}